\documentclass{article} %

\usepackage{etoolbox}
\newtoggle{neurips}
\togglefalse{neurips}
\newcommand{\neurips}[1]{\iftoggle{neurips}{#1}{}}
\newcommand{\arxiv}[1]{\iftoggle{neurips}{}{#1}}

\newif\ifsup\suptrue

\usepackage{comment}
\neurips{\usepackage[nonatbib,final]{neurips_2020}}
\usepackage[numbers]{natbib}
\usepackage[utf8]{inputenc} %
\usepackage[T1]{fontenc}    %
\neurips{\usepackage{hyperref}}       %
\usepackage{url}            %
\usepackage{booktabs}       %
\usepackage{amsfonts}       %
\usepackage{nicefrac}       %
\usepackage{microtype}      %
\usepackage{amsmath}
\usepackage{amssymb}
\usepackage{wrapfig}
\usepackage{amsthm}
\usepackage{mathtools}
\usepackage[dvipsnames]{xcolor}
\usepackage{bm}
\usepackage{dsfont}
\usepackage{mathrsfs}
\usepackage[linesnumbered,ruled,vlined]{algorithm2e} 
\SetKw{KwInit}{Initialize}

\usepackage[scaled=.88]{helvet}
\usepackage{xspace}

\arxiv{
\usepackage[letterpaper, left=1in, right=1in, top=1in, bottom=1in]{geometry}
\usepackage[colorlinks=true, linkcolor=blue!70!black, citecolor=blue!70!black,urlcolor=black,breaklinks=true]{hyperref}
\usepackage{breakcites}

 \usepackage{natbib}
 \bibliographystyle{plainnat}
 \bibpunct{(}{)}{;}{a}{,}{,}

}

\usepackage[capitalize,nameinlink]{cleveref}

\newcommand{\pref}[1]{\cref{#1}}

\DeclarePairedDelimiter{\abs}{\lvert}{\rvert} %
\DeclarePairedDelimiter{\brk}{[}{]}
\DeclarePairedDelimiter{\crl}{\{}{\}}
\DeclarePairedDelimiter{\prn}{(}{)}
\DeclarePairedDelimiter{\nrm}{\|}{\|}
\DeclarePairedDelimiter{\tri}{\langle}{\rangle}

\DeclarePairedDelimiter{\ceil}{\lceil}{\rceil}
\DeclarePairedDelimiter{\floor}{\lfloor}{\rfloor}

\DeclareMathOperator{\En}{\mathbb{E}}
\DeclareMathOperator*{\argmin}{argmin} %
\DeclareMathOperator*{\argmax}{argmax}             
\def\ddefloop#1{\ifx\ddefloop#1\else\ddef{#1}\expandafter\ddefloop\fi}
\def\ddef#1{\expandafter\def\csname bb#1\endcsname{\ensuremath{\mathbb{#1}}}}
\ddefloop ABCDEFGHIJKLMNOPQRSTUVWXYZ\ddefloop
\def\ddefloop#1{\ifx\ddefloop#1\else\ddef{#1}\expandafter\ddefloop\fi}
\def\ddef#1{\expandafter\def\csname b#1\endcsname{\ensuremath{\mathbf{#1}}}}
\ddefloop ABCDEFGHIJKLMNOPQRSTUVWXYZ\ddefloop
\def\ddef#1{\expandafter\def\csname c#1\endcsname{\ensuremath{\mathcal{#1}}}}
\ddefloop ABCDEFGHIJKLMNOPQRSTUVWXYZ\ddefloop
\def\ddef#1{\expandafter\def\csname h#1\endcsname{\ensuremath{\widehat{#1}}}}
\ddefloop ABCDEFGHIJKLMNOPQRSTUVWXYZ\ddefloop
\def\ddef#1{\expandafter\def\csname hc#1\endcsname{\ensuremath{\widehat{\mathcal{#1}}}}}
\ddefloop ABCDEFGHIJKLMNOPQRSTUVWXYZ\ddefloop
\def\ddef#1{\expandafter\def\csname t#1\endcsname{\ensuremath{\widetilde{#1}}}}
\ddefloop ABCDEFGHIJKLMNOPQRSTUVWXYZ\ddefloop
\def\ddef#1{\expandafter\def\csname tc#1\endcsname{\ensuremath{\widetilde{\mathcal{#1}}}}}
\ddefloop ABCDEFGHIJKLMNOPQRSTUVWXYZ\ddefloop

\newcommand{\ls}{\ell}

\newcommand{\veps}{\varepsilon}

\newcommand{\ldef}{\vcentcolon=}

\usepackage{xpatch}
\xpatchcmd{\proof}{\itshape}{\normalfont\proofnameformat}{}{}
\newcommand{\proofnameformat}{\bfseries}

\makeatletter
\newcommand{\DeclareMathActive}[2]{%
  \expandafter\edef\csname keep@#1@code\endcsname{\mathchar\the\mathcode`#1 }
  \begingroup\lccode`~=`#1\relax
  \lowercase{\endgroup\def~}{#2}%
  \AtBeginDocument{\mathcode`#1="8000 }%
}
\newcommand{\std}[1]{\csname keep@#1@code\endcsname}
\patchcmd{\newmcodes@}{\mathcode`\-\relax}{\std@minuscode\relax}{}{\ddt}
\AtBeginDocument{\edef\std@minuscode{\the\mathcode`-}}
\makeatother
\DeclareMathActive{*}{\ast}
\renewcommand{\ast}{\star}

\newcommand{\pstar}{p^{\star}}

\newcommand{\mathand}{\quad\text{and}\quad}

\newcommand{\mstar}{m^{\star}}
\newcommand{\thetahat}{\hat{\theta}}
\newcommand{\yhat}{\hat{y}}

\newcommand{\lossrange}{\brk{-1,+1}}

\newcommand{\mate}{\mathbf{e}}

\newcommand{\trn}{\top}

\newcommand{\psdgt}{\succ}

\newcommand{\fstar}{f^{\star}}

\newcommand{\alg}{\mathsf{Alg}_{\mathrm{Sq}}}

\newcommand{\indic}{\mathbb{I}}

\newcommand{\istar}{i^{\star}}

\newcommand{\poly}{\mathrm{poly}}

\newcommand{\astar}{a^{\star}}

\newcommand{\mainalg}{\textup{\textsf{SquareCB}}\xspace}

\newcommand{\bigoh}{\cO}
\newcommand{\bigoht}{\tilde{\cO}}
\newcommand{\bigom}{\Omega}
\newcommand{\bigomt}{\tilde{\Omega}}

\newcommand{\gfilt}{\mathfrak{F}}

\newcommand{\base}{\mathsf{Base}}
\newcommand{\master}{\mathsf{Master}}

\usepackage{arrayjobx}
\usepackage{multido}
\newarray\Eqnotes

\newcounter{equationnotes}
\stepcounter{equationnotes}
\newcounter{writtennotes}
\stepcounter{writtennotes}

\newcommand{\eqnote}[1]{\Eqnotes(\theequationnotes)={#1}(\alph{equationnotes})\stepcounter{equationnotes}}
\newcommand{\writeeqnotes}{\addtocounter{equationnotes}{-\thewrittennotes}\multido{\iActor=\thewrittennotes+1}{\value{equationnotes}}{%
(\alph{writtennotes}) \Eqnotes(\iActor)\stepcounter{writtennotes}}\setcounter{equationnotes}{\thewrittennotes}\setcounter{writtennotes}{\theequationnotes}}
\newcommand{\resetnotes}{\setcounter{equationnotes}{1}\setcounter{writtennotes}{1}}
\newcommand{\topnote}[2]{\overset{\eqnote{#2}}{#1}}

\newcommand{\conv}{\operatorname{conv}}

\newcommand{\mfa}{\mathfrak{a}} 

\usepackage[textsize=tiny]{todonotes}
\newcommand{\todoj}[2][]{\todo[size=\scriptsize,color=red!20!white,#1]{JZ: #2}}
\newcommand{\todocg}[2][]{\todo[size=\scriptsize,color=blue!20!white,#1]{CG: #2}}

\definecolor{dkblue}{cmyk}{1,.54,.04,.19}
    \hypersetup{
      bookmarks=true,         %
      unicode=false,          %
      pdftoolbar=true,        %
      pdfmenubar=true,        %
      pdffitwindow=false,     %
      pdfstartview={FitH},    %
      pdftitle={Adapting to Misspecification in Contextual Bandits},    %
      pdfauthor={},     %
      pdfnewwindow=true,      %
      colorlinks=true,        %
      linkcolor=blue!70!black,       %
      citecolor=blue!70!black,       %
      filecolor=blue!70!black,       %
      urlcolor=blue!70!black,        %
}
\usepackage{pgfplots}

\theoremstyle{plain}
\newtheorem{theorem}{Theorem}
\newtheorem{proposition}{Proposition}
\newtheorem{corollary}{Corollary}

\theoremstyle{definition}
\newtheorem{lemma}{Lemma}
\newtheorem{definition}{Definition}
\newtheorem{example}{Example}
\newtheorem{assumption}{Assumption}

\newtheorem*{theorem*}{Theorem}
\theoremstyle{remark}

\newcommand{\tr}{\operatorname{tr}}
\newcommand{\supp}{\operatorname{supp}}
\newcommand{\diam}{\operatorname{diam}}

\newcommand{\stab}{\operatorname{stab}}
\newcommand{\spann}{\operatorname{span}}
\renewcommand{\mid}{\,|\,}

\newcommand{\R}{\mathbb{R}}
\newcommand{\E}{\mathbb{E}}

\newcommand{\ip}[1]{\langle #1 \rangle}

\newcommand{\ber}{\mathrm{Ber}}
\newcommand{\set}[1]{\left\{ #1\right\}}
\newcommand{\norm}[1]{\Vert #1 \Vert}
\newcommand{\Reg}[1][]{\mathsf{Reg}_{#1}(T)}
\newcommand{\RegU}[1][]{\mathsf{Reg}^{#1}_{\mathrm{Unw}}(T)} %
\newcommand{\RegSquare}{\mathsf{Reg}_{\mathrm{Sq}}(T)}
\newcommand{\RegImp}[1][]{\mathsf{Reg}^{#1}_{\mathrm{Imp}}(T)}

\newcommand{\abelong}{\textsf{IGW}}
\newcommand{\logbarrier}{\textsf{log-barrier}}
\newcommand{\logdetbarrier}{\textsf{logdet-barrier}}

\newcommand{\squareCB}{\textup{\textsf{SquareCB}}\xspace}
\newcommand{\infalg}{\squareCB.\textup{\textsf{Lin}}\xspace}
\newcommand{\impSquareCB}{\textup{\textsf{SquareCB.Lin+}}\xspace}
\newcommand{\infSquareCB}{\textup{\textsf{SquareCB.Lin}}\xspace}

\newcommand{\corral}{\textsf{CORRAL}\xspace}

\newcommand{\SquareAlg}{\mathsf{Alg}_{\mathrm{Sq}}}

\newcommand{\dbar}{d_{\mathrm{avg}}}

\usepackage{inconsolata}

\let\epsilon\varepsilon
\newcommand{\ignore}[1]{}

\title{Adapting to Misspecification in Contextual Bandits}

\neurips{
\author{%
   Dylan J. Foster\thanks{Massachusetts Institute of Technology.}\\
   \texttt{dylanf@mit.edu}
  \And
  Claudio Gentile\thanks{Google Research.}\\
  \texttt{cgentile@google.com}
  \And
  Mehryar Mohri\footnotemark[2]\;\thanks{Courant Institute of Mathematical Sciences.}\\
  \texttt{mohri@google.com}
  \And
  Julian Zimmert\footnotemark[2]\\
  \texttt{zimmert@google.com}
}
}
\arxiv{
\author{%
   Dylan J. Foster\thanks{Massachusetts Institute of Technology.}\\
   {\small\texttt{dylanf@mit.edu}}
  \and
  Claudio Gentile\thanks{Google Research.}\\
  {\small\texttt{cgentile@google.com}}
  \and
  Mehryar Mohri\footnotemark[2]\;\thanks{Courant Institute of Mathematical Sciences.}\\
  {\small\texttt{mohri@google.com}}
  \and
  Julian Zimmert\footnotemark[2]\\
  {\small\texttt{zimmert@google.com}}
}
\date{}
}

\begin{document}

\maketitle

\begin{abstract}
A major research direction in contextual bandits is to
develop algorithms that are computationally efficient, yet support
flexible, general-purpose function approximation. Algorithms based on
modeling rewards have shown strong empirical performance, but
typically require a well-specified model, and can fail when this
assumption does not hold. Can we design algorithms that are efficient
and flexible, yet degrade gracefully in the face of model
misspecification? We introduce a new family of oracle-efficient algorithms for $\varepsilon$-misspecified contextual bandits
that adapt to unknown model misspecification---both for finite and
infinite action settings. Given access to an
\emph{online oracle} for square loss regression, our algorithm attains
optimal regret and---in particular---optimal dependence on the
misspecification level, with \emph{no prior knowledge}. Specializing
to linear contextual bandits with infinite actions in $d$ dimensions,
we obtain the first algorithm that achieves the optimal
$\bigoht(d\sqrt{T} + \varepsilon\sqrt{d}T)$ regret bound for unknown misspecification level $\veps$.

On a conceptual level, our results are enabled by a new
optimization-based perspective on the regression oracle reduction framework of
\citet{FR20}, which we anticipate will find broader use.

\end{abstract}

\section{Introduction}

The contextual bandit is a sequential decision making problem that is widely deployed in practice across applications including health services \citep{TM17},
online advertisement \citep{li2010contextual,AbeBiermannLong2003}, and
recommendation systems \citep{ABC16}. At each
round, the learner observes a feature vector (or \emph{context}) and
an action set, then selects an action and receives a loss for the action selected. To facilitate generalization across contexts, the
learner has access to a family of policies (e.g., linear models or neural networks) that map contexts to actions. The objective of the learner is 
to achieve cumulative loss close to that of the optimal policy in hindsight.

To develop efficient, general purpose algorithms, a common approach to contextual bandits is to reduce the problem to supervised learning primitives such as classification and regression
\citep{langford_epoch-greedy_2008,DHK11,A+12,AHK14,SyKrSch16,ABC16,LWA18}.
Recently, \citet{FR20} introduced \mainalg, which provides the first optimal and efficient reduction from
$K$-armed contextual bandits to square loss regression, and can be applied whenever the learner has access to a well-specified model for the loss function (``realizability''). In light of this result, a natural question is whether this approach can be generalized beyond the realizable setting and, more ambitiously, whether we can shift from realizable to misspecified models \emph{without prior knowledge} of the amount of misspecification. A secondary question, which is relevant for practical applications, is whether the approach can be generalized to large or infinite action spaces. This is precisely the setting we study in the present paper, where the action set is large or infinite, but where the learner has a ``good'' feature representation available---that is, up to some {\em unknown} amount of misspecification.

Adequately handling misspecification has been a subject of intense recent interest even for the simple special case of linear contextual bandits. \citet{DKW19} questioned whether ``good'' is indeed enough, that is, whether we can learn efficiently even without realizability. \citet{LS19feature} gave a
positive answer to this question---provided the misspecification level
$\veps$ is known in advance---and showed that the price of misspecification for linear contextual bandits scales as $\veps\sqrt{d}T$, where $d$ is the dimension and $T$ is the time horizon. However, they left the adapting to unknown misspecification as an open question.
\paragraph{Our results.}
We provide an affirmative answer to all of the questions above.
We generalize \mainalg to infinite action sets, and use this strategy to adapt to unknown misspecification by combining it with a \emph{bandit model selection} procedure akin to that of \citet{ALNS17}. Our algorithm is oracle-efficient, and adapts to misspecification efficiently and optimally whenever it has access to an online oracle for square loss regression. When specialized to linear contextual bandits, it resolves the question of \citet{LS19feature}.

On the technical side, a conceptual contribution of our work is to show that one can view the action selection scheme used by \mainalg as an approximation to a log-barrier regularized optimization problem, which paves the way for a generalization to infinite action spaces. Another byproduct of our results is a generalization of the original 
\corral algorithm \citep{ALNS17} for combining bandit algorithms, which is simple, flexible, and enjoys improved logarithmic factors.

\subsection{Related Work}

Contextual bandits are a well-studied problem, and misspecification in bandits and reinforcement learning has received much attention in recent years. Below we discuss a few results closely related to our own.

For linear bandits in $d$ dimensions, \citet{LS19feature} gave an algorithm with regret $\bigoh(d\sqrt{T} +\veps\sqrt{d}T)$, and left adapting to unknown misspecification for changing action sets as an open problem. Concurrent work of \citet{PPARZLS20} addresses this problem for the special case where contexts and action sets are stochastic, and also leverages \corral-type aggregation of contextual bandit algorithms. Our results resolve this question in the general, fully adversarial setting.

\arxiv{
  Within the literature general-purpose contextual bandit algorithms, our approach builds on a recent line of research that provides reductions to (online/offline) square loss regression \citep{foster2018practical,FR20,SX20,xu2020upper,foster2020instance}. In particular, our work builds on and provides a new perspective on the online regression reduction of \citet{FR20}. 
The infinite-action setting we consider was introduced in \citet{FR20}, but algorithms were only given for the special case where the action set is the sphere; our work extends this to arbitrary action sets. Concurrent work of \citet{xu2020upper} gives a reduction to offline oracles that also accommodates infinite action sets. These results are not strictly comparable: An online oracle can always be converted to an offline oracle through online-to-batch conversion and hence is a stronger assumption, but when an online oracle \emph{is} available, our algorithm is more efficient. In addition, by working with online oracles, we support adversarially chosen contexts.

It bears mentioning that misspecification in contextual bandits can be formalized in many ways, some of which go beyond the setting we consider. One line of work reduces stochastic contextual bandits to oracles for cost-sensitive classification \citep{langford_epoch-greedy_2008,DHK11,A+12,AHK14}. These results are agnostic, meaning they make no assumption on the model, and in particular do not require realizability. However, in the presence of misspecification, this type of guarantee is somewhat different from what we provide here: rather than giving a bound on regret to the true optimal policy, these results give bounds on the regret to the best-in-class policy. Another line of work considers a model in which the feedback received
by the learning algorithm at each round may be arbitrarily corrupted
by an adaptive adversary \citep{seldin2014one,lmp18,gkt19,BKS20}. Typical results for
this setting incur additive error $\bigoh(C)$, where $C$ is the
total number of corrupted rounds. While this model was originally
considered in the context of non-contextual stochastic bandits, it has recently been
extended to Gaussian process bandit optimization, which is closely
related to the contextual bandit setting, though these results only
tolerate $C\leq\sqrt{T}$. While these results are complementary to our own, we mention in passing that our notion of misspecification satisfies
$\veps\leq{}\sqrt{C/T}$, and hence our main theorem
(\pref{thm:main}) achieves additive error $\sqrt{CT}$ for this
corrupted setting (albeit, only with an oblivious adversary).
}

\neurips{
Within the literature general-purpose contextual bandit algorithms, our approach builds on a recent line of research that provides reductions to (online/offline) square loss regression \citep{foster2018practical,FR20,SX20,xu2020upper,foster2020instance}.

Besides the standard references on oracle-based agnostic contextual bandits (e.g., \cite{langford_epoch-greedy_2008,DHK11,A+12,AHK14}), $\epsilon$-misspecification is somewhat related to the recent stream of literature on bandits with adversarially-corrupted feedback~\cite{lmp18,gkt19,BKS20}. See the discussion in \ifsup Appendix \ref{sec:related_work}\else the supplementary material\fi.

}

\section{Problem Setting}

We consider the following contextual bandit protocol: At each round $t=1,\dots,T$, the learner first observes a context $x_t\in\cX$ and an action set $\cA_t\subseteq \cA$, where $\cA \subseteq \R^d$ is a compact action space; for simplicity, we assume throughout that $\cA = \{a \in \R^d \colon\norm{a}\leq 1\}$, but place no restriction on $\prn*{\cA_t}_{t=1}^{T}$. Given the context and action set, the learner chooses action $a_t\in\cA_t$, then observes a stochastic loss $\ell_t\in\lossrange$ depending on the action selected. We assume that the sequence of context vectors $x_t$ and the associated sequence of action sets $\cA_t$ are generated by an oblivious adversary. 

Let $\mu(a,x)\ldef{}\En\brk*{\ls_t\mid{}x_t=x,a_t=a}$ denote the mean loss function, which we assume to be time-invariant, and which is unknown to the learner. We adopt a semi-parametric approach to modeling the losses, in which $\mu(a,x)$ is modelled as (approximately) linear in the action $a$, but can depend on the context $x$ arbitrarily \citep{FR20,xu2020upper,chernozhukov2019semi}. In particular, we assume the learner has access to a class of functions $\cF \subseteq{}\{f\colon \cX \rightarrow \R^d \}$ where for each $f\in\cF$, $\tri*{a,f(x)}$ is a prediction for the value of $\mu(a,x)$. In the well-specified/realizable setting, one typically assumes that there exists some $\fstar\in\cF$ such that $\mu(a,x)=\tri*{a,f^{\star}(x)}$. In this paper, we make no such assumption, but the performance of our algorithms will depend on how far this is from being true. In particular, we measure performance of the learner in terms of pseudoregret $\Reg$ against the best unconstrained policy:
\begin{align*}
    \textstyle\Reg \ldef \E\brk*{\sum_{t=1}^T\mu(a_t,x_t)-\inf_{a\in\cA_t}\mu(a,x_t)}.
\end{align*}
Here, and throughout the paper, $\En\brk*{\cdot}$ denotes expectation with respect to both the randomized choices of the learner and the stochastic realization of the losses $\ell_t$.

This setup encompasses the finite-arm contextual bandit setting with $K$ arms by taking $\cA_t=\crl*{\mate_1,\ldots,\mate_K}$. Another important special case is the well-studied linear contextual bandit setting, where $\cF$ consists of constant vector-valued functions that do not depend on $\cX$. Specifically, for any $\Theta\subseteq\bbR^{d}$, take $\cF=\crl*{x\mapsto\theta\mid{}\theta\in\Theta}$. In this case, the prediction $\ip{a,f(x)}$ simplifies to $\ip{a,\theta}$, which a static linear function of the action. This special case recovers the most widely studied version of the linear contextual bandit problem \citep{AL99,Aue02,AST11,CLRS11,APS12,AG13,CG13}, as well as Gaussian process extensions \citep{skks10,KO11,DKC13,SGBK15}.

\subsection{Misspecification}
As mentioned above, contextual bandit algorithms based on modeling losses typically rely on the assumption of a \emph{well-specified model} (or, ``realizability''): That is, existence of a function $f^* \in \cF$ such that $\mu(a,x)=\ip{a,f^*(x)}$ for all $a \in \cA$ and $x \in \cX$ \citep{CLRS11,AST11,A+12,foster2018practical}. Since the assumption of \emph{exact} realizability does not typically hold in practice, a more recent line of work has begun to investigate algorithms for misspecified models.
In particular, \citet{CG13,GCG17,LS19feature,FR20,ZLKB20} consider a uniform $\epsilon$-misspecified setting in which
\begin{align}
\label{eq:old definition}
\textstyle\inf_{f\in \cF}\sup_{a \in \cA,x\in\cX}|\mu(a,x)-\ip{a,f(x)}| \leq \epsilon,
\end{align}
for some misspecification level $\epsilon >0$. Notably, \citet{LS19feature} show that for linear contextual bandits, regret must grow as $\bigom(d\sqrt{T}+\veps\sqrt{d}T)$. Since $d\sqrt{T}$ is the optimal regret for a well-specified model, $\veps\sqrt{d}T$ may be thought of as the price of misspecification.

We consider a weaker notion of misspecification. Given a sequence $S = (x_1,\cA_1), \ldots, (x_T,\cA_T)$ of context-action set pairs, we define the average misspecification level $\veps_T(S)$ as
\begin{align}
\label{eq:eps definition}
    \textstyle\veps_T(S) \ldef \inf_{f\in\cF}\prn*{\frac{1}{T}\sum_{t=1}^{T}\sup_{a\in\cA_t}(\ip{a,f(x_t)}-\mu(a,x_t))^2}^{1/2}.
\end{align}
This quantity measures the misspecification level for the specific sequence $S$ at hand, and hence offers tighter guarantees than uniform misspecification. In particular, the uniform bound in \cref{eq:old definition} directly implies $\veps_T(S)\leq\epsilon$ for all $S$ in \cref{eq:eps definition}, and $\veps_T(S)=0$ whenever the model is well-specified.

We provide regret bounds that optimally adapt to $\veps_T(S)$ for any given realization of the sequence $S$, with no prior knowledge of the misspecification level. The issue of adapting to unknown misspecification has not been addressed even for the stronger uniform notion \eqref{eq:old definition}. Indeed, previous efforts typically use prior knowledge of $\veps$ to encourage conservative exploration when misspecification is large; see \citet[Appendix E]{LS19feature}, \citet[Section 5.1]{FR20}, \citet[Section 4.2]{CG13}, and \citet{ZLKB20} for examples. Naively adapting such schemes using, e.g., doubling tricks, presents difficulties because the quantities in \pref{eq:old definition} and \pref{eq:eps definition} do not appear to be estimable without knowledge of $\mu$.

\subsection{Regression Oracles}

Following \citet{FR20}, we assume access to an \emph{online regression oracle} $\alg$, which is simply an algorithm for sequential prediction with the square loss, using $\cF$ as a benchmark class. Concretely, the oracle operates under the following protocol. At each round $t \in [T]$, the algorithm receives a context $x_t \in \cX$, outputs a prediction $\hat y_t \in \R^d$ (in particular, we interpret $\ip{a,\hat y_t}$ as the predicted loss for action $a$), then observes an action $a_t\in\cA$ and loss $\ls_t\in\lossrange$ and incurs error $(\ip{a_t,\hat y_t}-\ell_t)^2$. Formally, we make the following assumption.\footnote{As in \citet{FR20}, the \emph{square loss} itself does not play a crucial role, and can be replaced by any loss that is strongly convex with respect to the learner's predictions.}
\begin{assumption}
  \label{ass:regression}
  The regression oracle $\alg$ guarantees that for any (potentially adaptively chosen) sequence $\crl{(x_t,a_t,\ell_t)}_{t=1}^T$,
\begin{align*}
\textstyle\sum_{t=1}^T(\ip{a_t,\hat y_t}-\ell_t)^2 -\inf_{f\in\cF}\sum_{t=1}^{T}(\ip{a_t,f(x_t)}-\ell_t)^2 \leq \RegSquare\,,
\end{align*}
for some (non-data-dependent) function $\RegSquare$.
\end{assumption}
For the finite-action setting, this definition coincides with that of \citet{FR20}. To simplify the presentation of our results, we assume throughout the paper that $\RegSquare$ is a non-decreasing function of $T$.

While this type of oracle suffices for all of our results, our algorithms are stated more naturally in terms of a stronger notion of oracle that supports \emph{weighted} online regression. In this model, we follow the same protocol as in \pref{ass:regression}, except that at each time $t$, the regression oracle observes a weight $w_t\geq0$ at the same time as the context $x_t$, and the error incurred is given by $w_t\cdot(\tri{a_t,\yhat_t}-\ls_t)^{2}$. For technical reasons, we allow the oracle for this model to be randomized. We make the following assumption.
\begin{assumption}
  \label{ass:weighted regression}
  The weighted regression oracle $\alg$ guarantees that for any (potentially adaptively chosen) sequence $\crl{(w_t,x_t,a_t,\ell_t)}_{t=1}^T$, 
\begin{align*}
\textstyle\E\left[\sum_{t=1}^Tw_t(\ip{a_t,\hat y_t}-\ell_t)^2 - \inf_{f\in\cF}\sum_{t=1}^{T}w_t(\ip{a_t,f(x_t)}-\ell_t)^2\right] \leq \E\brk[\big]{\max_{t\in[T]}w_t}\cdot\RegSquare\,,
\end{align*}
for some (non-data-dependent) function $\RegSquare$, where the expectation is taken with respect to the oracle's randomization.
\end{assumption}
We show in \ifsup \cref{app:oracles} (\pref{alg:oracle}) \else the supplementary material\fi that any unweighted regression oracle satisfying \cref{ass:regression} can be transformed into a randomized oracle for weighted regression that satisfies \cref{ass:weighted regression}, with no overhead in runtime. Hence, to simplify exposition, for the remainder of the paper we state our results in terms of weighted regression oracles satisfying \pref{ass:weighted regression}.

Online regression is a well-studied problem, and efficient algorithms are known for many standard function classes. One example, which is important for our applications, is the case where $\cF$ is linear.
\begin{example}[Linear Models]
  \label{ex:ridge regression}
  Suppose $\cF=\crl*{x\mapsto{}\theta\mid{}\theta\in\Theta}$, where $\Theta\subseteq\bbR^{d}$ is a convex set with $\nrm*{\theta}\leq{}1$. Then the online Newton step algorithm \citep{hazan2007logarithmic}
satisfies \cref{ass:regression} with $\RegSquare = \cO(d\log(T))$ and---via reduction \ifsup(\pref{alg:oracle})\fi---can be augmented to satisfy \cref{ass:weighted regression}.
\end{example}

Further examples include kernels \citep{VKMFC13}, generalized linear models \citep{KKSK11}, and standard nonparametric classes \citep{gaillard2015chaining}. We refer to \citet{FR20} for a more comprehensive discussion.

\paragraph{Additional notation.}
For a set $X$, we let $\Delta(X)$ denote the set of all probability distributions over $X$. 
If $X$ is continuous, we restrict $\Delta(X)$ to distributions with \emph{countable} support.
We let $\nrm*{x}$ denote the euclidean norm for $x\in\bbR^{d}$. For any positive definite matrix $H\in\bbR^{d\times{}d}$, we denote the induced norm on $x \in \R^d$ by $\norm{x}^2_H=\ip{x,Hx}$. For functions
	$f,g:X\to\bbR_{+}$, we write $f=\bigoh(g)$ if there exists some constant
	$C>0$ such that $f(x)\leq{}Cg(x)$ for all $x\in{}X$. We write $f=\bigoht(g)$ if $f=\bigoh(g\max\crl*{1,\mathrm{polylog}(g)})$, and define $\bigomt(\cdot)$ analogously.

        For each $f\in\cF$, we let $\pi_f(\cdot,\cdot)$ denote the \emph{induced policy}, whose action at time $t$ is given by $\pi_f(x_t,\cA_t) \ldef\argmin_{a\in\cA_t}\ip{a,f(x_t)}$.

\section{Adapting to Misspecification: An Oracle-Efficient Algorithm}
\label{sec:main}

We now present our main result: an efficient reduction from contextual
bandits to online regression that adapts to unknown misspecification
$\veps_T(S)$ and supports infinite action sets. Our main theorem is as
follows.
\begin{theorem}
  \label{thm:main}
There exists an efficient reduction which, given access to a weighted regression oracle $\SquareAlg$
  satisfying \pref{ass:weighted regression}, guarantees that for all sequences $S = (x_1,\cA_1),\ldots, (x_T,\cA_T)$,
\begin{align*}
    \Reg = \cO\left(\sqrt{dT\RegSquare \log(T)}+ \epsilon_T(S)\sqrt{d}T\right).
\end{align*}
\end{theorem}
The algorithm has two main building blocks: First, we extend the reduction of
\citet{FR20} to infinite action sets via a new optimization-based
perspective and---in particular---show that the resulting algorithm has favorable
dependence on misspecification level when it is known in
advance. Then, we combine this reduction with a scheme that
aggregates multiple instances of the algorithm to adapt to unknown
misspecification. When the time required for a single query to
$\SquareAlg$ is $\cT_{\SquareAlg}$, the per-step runtime of our algorithm is $\bigoht\prn*{\cT_{\SquareAlg}+\abs{\cA_t}\cdot\poly(d)}$.

As an application, we solve an open problem
recently posed by \citet{LS19feature}: we exhibit an efficient
algorithm for infinite-action linear contextual bandits which
optimally adapts to unknown misspecification.
\begin{corollary}
\label{cor:main}
Let $\cF=\{x\mapsto{}\theta\mid\theta \in \R^d,\nrm{\theta} \leq 1\}$.
Then there exists an efficient algorithm that, for any sequence $S =
(x_1,\cA_1),\ldots, (x_T,\cA_T)$, satisfies
\begin{align*}
    \Reg = \cO\left(d\sqrt{T}\log(T) + \epsilon_T(S)\sqrt{d}T\right).
\end{align*}
\end{corollary}
This result immediately follows from \cref{thm:main} by invoking the online
Newton step algorithm as the regression oracle, as in \pref{ex:ridge regression}.
Modulo logarithmic factors, this bound coincides with the one achieved
by \citet{LS19feature} for the simpler non-contextual linear bandit
problem, for which the authors present a matching lower bound.

The remainder of this section is dedicated to proving
\pref{thm:main}. The roadmap is as follows. First, we revisit the
reduction from $K$-armed contextual bandits to online regression by
\citet{FR20} and provide a new optimization-based perspective. This new viewpoint leads to a natural generalization from
the $K$-armed case to the infinite action case.  We then provide an
aggregation-type procedure which combines multiple instances of this
algorithm to adapt to unknown misspecification, and finally put all
the pieces together to prove the main result. As an extension, we also
give a variant of the algorithm which enjoys improved bounds when the
action sets $\cA_t$ lie in low-dimensional subspaces of $\R^d$.

Going forward, we abbreviate
$\veps_T(S)$ to $\veps_T$ whenever the sequence $S$ is clear from context.

\subsection{Oracle Reductions with Finite Actions: An Optimization-Based Perspective}
\label{sec:k-armed}
A canonical special case of our setting is the finite-arm contextual
bandit problem, where
$\cA_t=\cK\ldef\{\mathbf{e}_1,\dots,\mathbf{e}_K\}$. For this setting,
\citet{FR20} proposed an efficient and optimal reduction called
\squareCB, which is displayed in \cref{alg:dylan}. At each step,
the algorithm queries the oracle $\SquareAlg$ with the current context $x_t$ and
receives a loss predictor $\thetahat_t\in\bbR^{K}$, where
$(\thetahat_t)_i$ predicts the loss of action $i$. The algorithm then
samples an action using an \emph{inverse gap weighting} (IGW) scheme introduced by
\citet{AL99}.
\begin{wrapfigure}[11]{r}{6.7cm}
\vspace{-0.15cm}
\begin{algorithm}[H]
\caption{\squareCB\\\citep{FR20}}
\label{alg:dylan}
\DontPrintSemicolon
\LinesNumberedHidden
\KwIn{Learning rate $\gamma$, horizon $T$.}
\KwInit{} Regression oracle $\SquareAlg$. \;
\For{$t= 1,\dots,T$}{
Receive context $x_t$.\;
Let $\hat\theta_t$ be the oracle's prediction for $x_t$.\;
Sample $I_t\sim \abelong(\hat\theta_t,\gamma)$.\;
Play $a_t=\mathbf{e}_{I_t}$ and observe loss $\ell_t$.\;
Update $\alg$ with 
$(x_t, a_t, \ell_t)$.
}
\end{algorithm}
\end{wrapfigure}
Specifically for parameter $\theta\in\bbR^{K}$ and learning
rate $\gamma>0$, we define $\abelong(\theta,\gamma)$ as
the distribution $p\in\Delta(\brk*{K})$ obtained by first selecting
any $i^*\in\argmin_{i\in[K]}\theta_{i}$, then defining
\begin{align}
    p_{i} = \begin{cases} \frac{1}{K+\gamma(\theta_{i} %
    -\theta_{i^*})},&\mbox{ if }i\neq i^*,\\
    1-\sum_{i'\neq i^*}p_{i},&\mbox{ otherwise.}
    \end{cases}\label{eq:abe-long}
\end{align}
By choosing $\gamma\propto\sqrt{KT/(\RegSquare+\veps_T)}$, one can
show that this algorithm
guarantees \[\Reg\leq\bigoh\prn*{\sqrt{KT\RegSquare}+\veps_T\sqrt{K}T}.\] Since this
approach is the starting point for our results, it will be useful to sketch the
proof. For $p\in\Delta(\cA)$, let $H_p\ldef\E_{a\sim p}\brk{aa^\trn}$
be the second moment matrix, and $\bar a_p \ldef \E_{a\sim p}[a]$ be the
mean action.
Let the sequence $S$ be fixed, and let $f^*\in\cF$ be any regression
function that attains the value of $\veps_T(S)$ in \pref{eq:eps
  definition}.\footnote{If the infimum is not obtained, it suffices to
  apply the argument that follows with a limit sequence.} With $a^*_t
\ldef\pi_{f^*}(x_t,\cA_t)$ and $\theta^*_t\ldef{}f^*(x_t)$, we have
\begin{align*}
    &\textstyle\E\left[\sum_{t=1}^T\mu(a_t,x_t)-\inf_{a\in\cA_t}\mu(a,x_t)\right]\\&\leq \E\left[\sum_{t=1}^T\ip{a_t-a^*_t,\theta^*_t} \right] + 2 \veps_T T\\
    &\textstyle=  \E\left[\sum_{t=1}^T\ip{\bar a_{p_t}-a^*_t,\theta^*} -\frac{\gamma}{4}\norm{\theta^*-\hat\theta_t}^2_{H_{p_t}} \right]+\E\left[\sum_{t=1}^T\frac{\gamma}{4}\norm{\theta^*-\hat\theta_t}^2_{H_{p_t}}\right] + 2 \veps_T T\,.
\end{align*}
The first expectation term above is bounded by $\bigoh(KT/\gamma)$,
which is established by showing that $\abelong(\thetahat,\gamma)$ is
an approximate solution to the per-round minimax problem
\neurips{\vspace{-0.5em}}
\begin{align}
    \label{eq:minimax}
     \min_{p\in\Delta(\cK)}\max_{\theta\in\mathbb{R}^K}\max_{a^*\in\cK}\ip{\bar a_p-a^*,\theta}-\frac{\gamma}{4}\norm{\hat\theta-\theta}^2_{H_p}\,,
\end{align}
with value $\bigoh(K/\gamma)$.
The second expectation term is bounded by
$\bigoh(\gamma\cdot(\RegSquare+\veps_T{}T))$, which follows readily from the definition of the square loss regret
in \pref{ass:regression} (see the proof of \pref{thm:impreg} for
details). Choosing $\gamma$ to balance the terms leads to the result.

As a first step toward generalizing this result to infinite actions,
we propose a new distribution that \emph{exactly} solves the minimax
problem \eqref{eq:minimax}. This distribution is the solution to a
dual optimization problem based on \emph{log-barrier} regularization, and
provides a principled approach to deriving contextual bandit reductions.
\begin{lemma}
\label{lem:logbar opt}
For any $\theta\in\R^K$ and $\gamma >0$, the unique minimizer of \cref{eq:minimax} coincides with the unique minimizer of the $\logbarrier(\theta,\gamma)$ optimization problem defined by
\begin{align}
  \textstyle\label{eq:logbarrier_form}
    \logbarrier(\theta, \gamma) = \argmin_{p\in\Delta([K])}\crl[\bigg]{\langle p, \theta\rangle -\frac{1}{\gamma}\sum_{a\in[K]}\log(p_a)} = \left(\frac{1}{\lambda + \gamma\theta_{i} }\right)_{i=1}^K\,,
\end{align}
where $\lambda$ is the unique value that ensures that the weights on
the right-hand side above sum to one.
\end{lemma}
The $\abelong$ distribution
is closely related to the $\logbarrier$ distribution: Rather than finding
the optimal Lagrange multiplier $\lambda$ that solves the $\logbarrier$
problem, the $\abelong$ strategy simply plugs in $\lambda=K-\gamma\min_{i'}\theta_{i'}$, then shifts
weight to $p_{i^{\star}}$ to ensure the distribution is
normalized. Since the $\logbarrier$ strategy solves the minimax problem
\pref{eq:minimax} exactly, plugging it into the results of
\citet{FR20} and \citet{SX20} in place of $\abelong$ leads to slightly
improved constants. More importantly, this new perspective leads to a
principled way to extend these reductions to infinite actions.
\subsection{Moving to Infinite Action Sets: The Log-Determinant Barrier}
\label{sec:linear}
\begin{wrapfigure}[13]{r}{6.7cm}
\vspace{-0.1cm}
\begin{algorithm}[H]
\caption{\infSquareCB}
\label{alg:logdet}
\DontPrintSemicolon
\LinesNumberedHidden
\KwIn{Learning rate $\gamma$, horizon $T$.}
\KwInit{} Regression oracle $\SquareAlg$. \;
\For{$t= 1,\dots,T$}{
Receive context $x_t$.\;
Let $\hat\theta_t$ be the oracle's prediction for $x_t$.\;
Play $a_t\sim \logdetbarrier(\hat\theta_t,\gamma;\cA_t)$.\;
Observe loss $\ell_t$.\;
Update $\alg$ with 
$(x_t, a_t, \ell_t)$.
}
\end{algorithm}
\vspace{-0.0in}
\end{wrapfigure}
To support infinite action sets, we replace the $\logbarrier$
distribution with a generalization based on the log-determinant function.
In order to state the result, let $\dim(\cA)$ denote the dimension of the smallest affine linear subspace that contains $\cA$.
When $\dim(\cA)<d$, we adopt the convention that the determinant
function $\det(\cdot)$ takes
the product of only the first $\dim(\cA)$ eigenvalues of the matrix in
its argument.\footnote{This convention ensures that the solution to
  the $\logdetbarrier$ problem is well-defined.}
We define the $\logdetbarrier$ distribution as follows.
\begin{definition}
For parameter $\theta \in \R^d$, action set $\cA\subset\R^d$, and
learning rate
$\gamma>0$, $\logdetbarrier(\theta,\gamma;\cA)$ is defined as the set
of solutions to
\begin{align}
    \textstyle\argmin_{p\in\Delta(\cA)}  \crl*{\ip{\bar a_p,\theta} -\gamma^{-1} \log\det (H_p-\bar a_p\bar a^T_p)}\,.
    \label{eq:logdet_barrier}
\end{align}
\end{definition}
In general, \pref{eq:logdet_barrier} does not admit a unique
solution; all of our results apply to \emph{any} minimizer. Our key
result is that the $\logdetbarrier$ distribution also solves a minimax
problem analogous to that of \pref{eq:minimax}.
\begin{lemma}
\label{lem:logdet bound}
 Any solution to $\logdetbarrier(\thetahat,\gamma;\cA)$ satisfies
 \begin{align}
 \textstyle\max_{\theta\in\mathbb{R}^d}\max_{a^*\in\cA}\ip{\bar a_p-a^*,\theta} - \frac{\gamma}{4}\norm{\hat\theta-\theta}^2_{H_p}\leq \gamma^{-1}\dim(\cA).\label{eq:key_inequality}
 \end{align}
\end{lemma}
By replacing the $\abelong$ distribution with the $\logdetbarrier$
distribution in \pref{alg:dylan}, we obtain an optimal reduction for
infinite action sets. This algorithm, which we call \infalg, is
displayed in \pref{alg:logdet}.
\begin{theorem}
  \label{thm:infinite}
  Given a regression oracle $\SquareAlg$
  that satisfies \pref{ass:regression}, \infalg with
  learning rate $\gamma\propto\sqrt{dT/(\RegSquare+\veps)}$ guarantees
  that for all sequences $S$ with $\veps_T(S)\leq\veps$,
\begin{align*}
    \textstyle\Reg = \cO\left(\sqrt{dT\RegSquare}+ \veps\sqrt{d}T\right)\,.
\end{align*}
\end{theorem}
The $\logdetbarrier$ optimization problem is closely related to the
D-optimal experimental design problem, as well as the John ellipsoid problem
\citep{KT90,TY07}; the latter corresponds to the case where $\theta=0$ in
\pref{eq:logdet_barrier} \citep{KY05}. By adapting specialized
optimization algorithms for these problems (in particular, a
Frank-Wolfe-type scheme), we can efficiently solve the
$\logdetbarrier$ problem.
\begin{proposition}
\label{prop:logdet_efficient}
An approximate minimizer for \eqref{eq:logdet_barrier} that achieves the same
regret bound up to a constant factor can be computed in time
$\tilde\cO\prn*{|\cA_t|\cdot\mathrm{poly}(d)}$ and memory
$\bigoht(\log\abs{\cA_t}\cdot\poly(d))$ per round.
\end{proposition}
The minimization algorithm, along with a full analysis for runtime and
memory complexity and  impact on the regret, is provided in \ifsup \cref{app:logdet solver}\else the supplementary material\fi.

\subsection{Adapting to Misspecification: Algorithmic Framework}
\label{sec:adaptive}
The regret bound for \infalg in \pref{thm:infinite} achieves optimal
dependence on the dimension and misspecification level, but
requires an a-priori upper bound on $\veps_T(S)$ to set the learning rate. We now turn our
attention to adapting to this parameter.

At a high level, our approach is to run multiple instances of \infalg,
each tuned to a different level of misspecification, then run an
aggregation procedure on top to learn the best instance. Specifically,
we initialize a collection of $M\ldef{}\floor{\log(T)}$ instances of
\cref{alg:logdet} in which the learning rate for instance $m$ is tuned
for misspecification level $\veps'_m \ldef \exp(-m)$ (that is, we
follow a geometric grid). It is straightforward to show that
there exists $m^{\star}\in\brk{M}$ such that the $\mstar$th instance
would enjoy optimal regret if we were to run it on the sequence
$S$. Of course, $\mstar$ is not known a-priori, so we run an aggregation (or,
``Corralling'') procedure to select the best
instance \citep{ALNS17}. This approach is, in general, not suitable for model selection,
since it typically requires prior knowledge of the optimal regret
bound to tune certain parameters appropriately
\citep{foster2019model}. Our conceptual insight is to show that adaptation to
misspecification is an exception to this rule, and offers a simple
setting where model selection for contextual bandits is possible.

\begin{wrapfigure}[12]{r}{8cm}
\vspace{-0.1cm}
\begin{algorithm}[H]
\caption{Corralling \\\citep{ALNS17}}
\label{alg:master}
\DontPrintSemicolon
\LinesNumberedHidden
\KwIn{Master algorithm $\master$, $T$}
\KwInit{$(\base_m)_{m=1}^M$ \;}
\For{$t= 1,\dots, T$}{
Receive context $x_t$.\;
Receive $A_t$, $q_{t,A_t}$ from $\master$.\;
Pass $(x_t,\cA_t, q_{t,A_t},\rho_{t,A_t})$ to $\base_{A_t}$.\; 
$\base_{A_t}$ plays $a_t$ and observes $\ell_t$.\;
Update $\master$ with $\tilde\ell_{t,A_t}=(\ell_t+1)$.\;
}
\end{algorithm}
\end{wrapfigure}
We use the aggregation scheme in \cref{alg:master}, which is a
generalization of the \corral algorithm of \citet{ALNS17}.

The algorithm is initialized with $M$ \emph{base} algorithms, and uses a
multi-armed bandit algorithm with $M$ arms as a \emph{master}
algorithm whose role is to choose the base algorithm to follow at
each round.

In more detail, the master algorithm maintains a distribution $q_t\in\Delta(\brk*{M})$ over the
base algorithms. At each round $t$, it samples an algorithm $A_t\sim q_t$
and passes the current context $x_t$ into this algorithm, as well as
the sampling probability $q_{t,A_t}$, and an importance weight
$\rho_{t,A_t}$, where we define $\rho_{t,m}\ldef{}1/\min_{s\leq
  t}q_{s,m}$ for each $m$.
At this point, the base algorithm $A_t$ selected by the master
executes a standard contextual bandit round:
Given the context $x_t$, it selects an arm $a_t$,
receives the loss $\ls_t$, and updates its internal state. Finally,
the master updates its state with the action-loss pair
$(A_t,\tilde \ell_{t,A_t})$, where $\tilde
\ell_{t,A_t}\ldef{}\ls_t+1$; for technical reasons related to our
choice of master algorithm, it is useful to
shift the loss by $1$ to ensure non-negativity.

Define the \emph{importance-weighted regret} for base algorithm $m$ as
\neurips{$\RegImp[m]\ldef\E\brk*{\sum_{t=1}^T\frac{\bbI\set{A_t=m}}{q_{t,m}}\left(\mu(a_t,x_t)-\inf_{a\in\cA_t}\mu(a,x_t)\right)}$}
\arxiv{\[\RegImp[m]\ldef\E\brk*{\sum_{t=1}^T\frac{\bbI\set{A_t=m}}{q_{t,m}}\left(\mu(a_t,x_t)-\inf_{a\in\cA_t}\mu(a,x_t)\right)},\]}
which is simply the pseudoregret incurred in the rounds where
\pref{alg:master} follows this base algorithm, weighted inversely
proportional to the probability that this occurs. It is
straightforward to show that for any choice for the master and base
algorithms, this scheme guarantees that
\begin{align}
    \textstyle\Reg %
    = \E\brk*{\sum_{t=1}^T\tilde\ell_{t,A_t}-\tilde\ell_{t,m^*}}+\RegImp[m^*]~
    ,\label{eq:corral_decomp}
\end{align}
where $\tilde \ell_{t,m}$ denotes the loss that the algorithm would have
suffered at round $t$ if the master algorithm had chosen $A_t=m$.  In
other words, the regret of \pref{alg:master} is equal to the regret
$\Reg[M]\ldef\E[\sum_{t=1}^T\tilde\ell_{t,A_t}-\tilde\ell_{t,m^*}]$ of the master algorithm, plus the importance-weighted regret
of the optimal base algorithm  $m^*$.

The difficulty in instantiating this general scheme lies in the fact
that the importance-weighted regret $\RegImp[m^*] $ of the optimal base
algorithm typically scales
with $\E[\rho_{T,m^*}^\alpha]\cdot\RegU[m^*]$, where
$\alpha\in[\frac{1}{2},1]$ is an algorithm-dependent parameter and
$\RegU[m]\ldef\E[\sum_{t=1}^T\bbI\set{A_t=m}\left(\mu(a_t,x_t)-\inf_{a\in\cA_t}\mu(a,x_t)\right)]$
denotes the unweighted regret of algorithm $m$. A-priori, the
$\E[\rho_{T,m^*}^\alpha]$ can be unbounded, leading to large regret.
The key to the analysis of \citet{ALNS17}, and the approach we follow
here, is to use a master algorithm
with \emph{negative regret} proportional to $\E[\rho_{T,m^*}^\alpha]$,
allowing to cancel this factor.

\begin{wrapfigure}[14]{r}{8cm}
\vspace{-1.15cm}
\begin{algorithm}[H]
\caption{\impSquareCB (for base alg. $m$)}
\label{alg:imp logdet}
\DontPrintSemicolon
\LinesNumberedHidden
\KwIn{$T$, $\RegSquare$}
\KwInit{} Weighted regression oracle $\alg$.\;
\For{$t= (\tau_1,\tau_2,\ldots)\subset [T]$}{
Receive context $x_t$ and $(q_{t,m},\rho_{t,m})$.\;
Set $\gamma_{t,m} =
\min\set{\frac{\sqrt{d}}{\veps'_m},\sqrt{dT/(\rho_{t,m}\RegSquare)}}$.\;
Set $w_t=\gamma_{t,m}/q_{t,m}$.\;
Compute oracle's prediction $\thetahat_t$ for $x_t, w_t$.\;
Sample $a_t\sim \logdetbarrier(\theta_t,\gamma_{t,m};\cA_t)$.\;
Play $a_t$ and observe loss $\ell_t$.\;
Update $\alg$ with $\prn*{w_t, x_t, a_t, \ell_t}$.
}
\end{algorithm}
\vspace{-0.0in}
\end{wrapfigure}

\subsubsection{Choosing the Base Algorithm}
As the first step towards instantiating the aggregation scheme above,
we specify the base algorithm. We use a modification to \infalg (denoted by $\impSquareCB$) based
on 
importance weighting, which is designed to ensure that the
importance-weighted regret in \pref{eq:corral_decomp} is
bounded. Pseudocode for the $m$th base algorithm is given in
\pref{alg:imp logdet}.

\impSquareCB proceeds as follows. Let the instance $m$ be fixed, and
let $Z_{t,m}=\indic\crl{A_t=m}$ indicate the event that this instance
is chosen to select an arm; note that we have $Z_{t,m}\sim \ber(q_{t,m})$
  marginally. When $Z_{t,m}=1$, instance $m$ receives $q_{t,m}$ and
  $\rho_{t,m}=\max_{s\leq t}q_{s,m}^{-1}$ from the master
  algorithm. The instance then follows the same update scheme as in the
  vanilla version of \infalg, except that i) it uses an adaptive
  learning rate $\gamma_{t,m}$, which is tuned based on $\rho_{t,m}$,
  and ii) it uses a weighted square loss regression oracle
  (as in \pref{ass:weighted regression}), with the weight $w_t$ set as a function
  of $\gamma_{t,m}$ and $q_{t,m}$.

The importance weighted regret $\RegImp[m]$ for this scheme is bounded as follows.
\begin{theorem}
\label{thm:impreg}
When invoked within \cref{alg:master} using a weighted regression oracle satisfying \cref{ass:weighted regression},
the importance-weighted regret for each instance $m\in\brk*{M}$ of \cref{alg:imp logdet} satisfies
\begin{align}
    \textstyle\RegImp[m] \leq
  \frac{3}{2}\E[\sqrt{\rho_{T,m}}]\sqrt{dT\RegSquare} +
  \left(\left(\frac{\veps'_m}{\veps_T}+\frac{\veps_T}{\veps'_m}
  \right) \sqrt{d}+2\right)\veps_T T.
  \label{eq:impreg}
\end{align}
\end{theorem}
The key feature of this regret bound is that only the leading term
involving $\RegSquare$ depends on the importance weights, not the
second misspecification term. This means that the optimal tuning for
the master algorithm will depend on $d$, $T$, and
$\RegSquare$, but not on $\veps_T$, which is critical to adapt without
prior knowledge of the misspecification. Another important feature is
that as long as $\veps'_m$ is within a
constant factor of $\veps_T$, the second term simplifies to
$\bigoh(\veps_T\sqrt{d}T)$ as desired.

\subsubsection{Improved Master Algorithms for Combining Bandit Algorithms}
It remains to provide a master algorithm for use within
\pref{alg:master}. While it turns out the master algorithm proposed in
\citet{ALNS17} suffices for this task, we go a step further and
propose a new master algorithm called \emph{$(\alpha,R)$--hedged
  FTRL}, which is simpler and enjoys slightly
improved regret, removing logarithmic factors. While this is not the
focus of the paper, we find it to be a useful secondary contribution because it 
provides a new approach to designing master algorithms for
bandit aggregation. We hope it will find use more broadly. 

The $(\alpha,R)$--hedged FTRL algorithm is parameterized by a regularizer and two scale parameters $\alpha\in(0,1)$ and $R>0$. We defer a precise definition and analysis to
\ifsup \cref{sec:hedged ftrl}\else supplementary material\fi, and state only the relevant result for our
aggregation setup here. 
We consider a special case of the $(\alpha,R)$--hedged FTRL algorithm that we call \emph{$(\alpha,R)$--hedged Tsallis-INF}, which
instantiates the framework using the Tsallis entropy as a regularizer
\citep{AB09,ALT15,ZiSe18}.  The key property of the algorithm is that
the regret with respect to a policy playing a fixed arm $m$ contains a
negative contribution proportional to $\rho_{T,m}^\alpha R$. The following result is a corollary of a
more general theorem, \ifsup \pref{thm:hedged ftrl} (\pref{sec:hedged ftrl})\else found in the supplementary material\fi.
\begin{corollary}
  \label{cor:Tsallis-INF}
  Consider the adversarial multi-armed bandit problem with $M$ arms
  and losses $\tilde\ell_{t,m} \in [0,2]$. For any $\alpha\in(0,1)$ and
  $R>0$, the $(\alpha,R)$--hedged Tsallis-INF algorithm with learning rate
  $\eta=\sqrt{1/(2T)}$ guarantees that for all $\mstar\in\brk*{M}$,
\begin{align}
    \textstyle\E\left[\sum_{t=1}^T\tilde\ell_{t,A_t}-\tilde\ell_{t,\mstar}\right] \leq 4\sqrt{2MT} + \E\left[\min\set{\frac{1}{1-\alpha},2\log(\rho_{T,\mstar })}M^\alpha -\rho_{T,\mstar}^\alpha \right]\cdot{}R\,.\label{eq:tsallis_regret}
\end{align}
\end{corollary}

\subsection{Putting Everything Together}
When invoked within \pref{alg:master}, $(\alpha,R)$-hedged Tsallis-INF
has a negative contribution to the cumulative regret which, for sufficiently large
$R$ and appropriate $\alpha$, can be used to offset the regret
incurred from importance-weighting the base algorithms. In particular,
$\prn*{\frac{1}{2},\frac{3}{2}\sqrt{dT\RegSquare}}$--hedged Tsallis-INF has
exactly the negative regret contribution needed to cancel the
importance weighting term in \pref{eq:impreg} if we use \impSquareCB
as the base algorithm. In more detail, we prove \pref{thm:main} by
combining the regret bounds for the master and base algorithms as follows.
\begin{proof}[Proof sketch for \cref{thm:main}]
  Using \pref{eq:corral_decomp}, it suffices to bound the regret of
  the bandit master $\Reg[M]$ and the importance-weighted regret
  $\RegImp[m^{\star}]$ for the optimal instance $m^*$.
By \cref{cor:Tsallis-INF}, using 
$\prn*{\frac{1}{2},\frac{3}{2}\sqrt{dT\RegSquare}}$--hedged Tsallis-INF as
the master algorithm gives
\begin{align*}
    \textstyle\Reg[M] \leq \cO\prn*{\sqrt{dT\RegSquare\log(T)}} - \frac{3}{2}\E[\sqrt{\rho_{T,m^*}}]\sqrt{dT\RegSquare}.
\end{align*}
Whenever the misspecification level is not
trivially small, the geometric grid ensures that there exists
$\mstar\in\brk{M}$ such that
$e^{-1}\veps_T\leq \veps'_{m^*}\leq \veps_T$. For this instance, \cref{thm:impreg} yields
\begin{align*}
    \textstyle\RegImp[m^*] \leq
  \frac{3}{2}\E[\sqrt{\rho_{T,m^*}}]\sqrt{dT\RegSquare} + \cO(\veps_T\sqrt{d}T).
\end{align*}
Summing the two bounds using \pref{eq:corral_decomp} completes the proof.
\end{proof}

\subsection{Extension: Adapting to the Average Dimension}
\label{sec:sparse}
A well-known application for linear contextual bandits is the problem of
online news article recommendation, where the context $x_t$ is taken
to be a feature vector containing information about the user, and each
action $a\in\cA_t$ is the concatenation of $x_t$ with a feature representation
for a candidate article (e.g., \citet{li2010contextual}). In this
and other similar applications, it is often the case that while
examples lie in a high-dimensional space, the true dimensionality $\dim(\cA_t)$ of the action set is small, so that $\dbar\ldef\frac{1}{T}\smash{\sum_{t=1}^T}\dim(\cA_t)\ll d$.
If we have prior knowledge of $\dbar$ (or an upper bound thereof), we can
exploit this low dimensionality for tighter regret. In fact, following
the proof of \cref{thm:impreg} and \cref{thm:main}, and bounding
$\smash{\sum_{t=1}^T}\dim(A_t)$ by $\dbar{}T$ instead of $dT$, it is fairly
immediate to show that
\cref{alg:master} enjoys improved regret $\Reg = \cO (\sqrt{\dbar{}
    \smash{T\RegSquare\log(T)}}+\veps_T\sqrt{\dbar{}}T)$, so long as $\dbar$ is replaced by
$d$ in the algorithm's various parameter settings.
Our final result shows that it is possible to adapt to unknown $\dbar{}$ and
unknown misspecification simultaneously. The key idea to apply a
doubling trick on top of \pref{alg:master}
\begin{theorem}
\label{thm:sparse}
There exists an algorithm that, under the same conditions as
\cref{thm:main}, satisfies 
$\Reg = \cO\left(\sqrt{\dbar T\RegSquare\log(T)}+\veps_T\sqrt{\dbar}T\right)$ without prior knowledge of $\dbar$ or $\veps_T$.
\end{theorem}
We remark that while the bound in \pref{thm:sparse} replaces
the $d$ factor in the reduction with the data-dependent quantity $\dbar$,
the oracle's regret $\RegSquare$ may itself still depend on $d$ unless
a sufficiently sophisticated algorithm is used.

\section{Discussion}

We have given the first general-purpose, oracle-efficient
algorithms that adapt to unknown model
misspecification in contextual bandits. For infinite-action linear contextual bandits, our
results yield the first optimal algorithms that adapt to unknown
misspecification with changing action sets. Our results suggest a
number of interesting questions:
\begin{itemize}
\item Can our optimization-based perspective lead to new oracle-based
  algorithms for more rich types of infinite action sets? Examples
  include nonparametric action sets and structured (e.g., sparse) linear action sets.
\item Can our reduction-based techniques be lifted to more
  sophisticated interactive learning settings such as reinforcement learning?
\end{itemize}
On the technical side, we anticipate that our new approach to reductions will find
broader use; natural extensions include reductions for offline oracles
\citep{SX20} and adapting to low-noise conditions
\citep{foster2020instance}.

\neurips{\clearpage}

\subsection*{Acknowledgements}
DF acknowledges the support of NSF TRIPODS grant \#1740751. We thank
Teodor Marinov and Alexander Rakhlin for
discussions on related topics.

\newpage
\neurips{\bibliographystyle{plainnat}}
\bibliography{all}

\newpage

\ifsup

\appendix

\neurips{
  \section{Additional Related Work}\label{sec:related_work}
Our work builds on and provides a new perspective on the online square loss oracle reduction of \citet{FR20}. 
The infinite-action setting we consider was introduced in \citet{FR20}, but algorithms were only given for the special case where the action set is the sphere; our work extends this to arbitrary action sets. Concurrent work of \citet{xu2020upper} gives a reduction to offline oracles for infinite action sets. This result is not strictly comparable: On one hand, an online oracle can always be converted to an offline oracle through online-to-batch conversion, but when an online oracle \emph{is} available our algorithm is significantly more efficient.

Misspecification in contextual bandits can be formalized in different ways that go beyond the setting we consider. First, we mention a long line of work which reduces stochastic contextual bandits to oracles for cost-sensitive classification \citep{langford_epoch-greedy_2008,DHK11,A+12,AHK14}. These results are agnostic, meaning they make no assumption on the model. However, in the presence of misspecification, the type of guarantee is somewhat different than what we provide here: rather than giving a bound on regret to the true optimal policy, these results give bounds on the regret to the best-in-class policy.

Another line of works consider a model in which the feedback received
by the learning algorithm at each round may be arbitrarily corrupted
by an adaptive adversary \cite{lmp18,gkt19,BKS20}. Typical results for
this setting pick up additive error $\bigoh(C)$, where $C$ is the
total number of corrupted rounds. While this model was original
introduced for non-contextual stochastic bandits, it has recently been
extended to Gaussian process bandit optimization, which is closely
related to the contextual bandit setting (though these results only
tolerate $C\leq\sqrt{T}$). While this is not the focus of our paper,
we mention in passing that our notion of misspecification satisfies
$\veps_T(S)\leq{}\sqrt{C/T}$, and hence our main theorem
(\pref{thm:main}) picks up additive error $\sqrt{CT}$ for this
corrupted setting (albeit, only with an oblivious adversary).
}

\section{Reducing Weighted to Unweighted Regression}
\label{app:oracles}
\arxiv{\begin{algorithm}[htpb]}
\neurips{\begin{algorithm}[t]}
\caption{Randomized reduction from weighted to unweighted online regression}
\label{alg:oracle}
\DontPrintSemicolon
\LinesNumberedHidden
\KwIn{Online regression oracle $\alg$ satisfying \pref{ass:regression}.}
\KwInit{$w_{\max} \leftarrow 0$} \;
\For{$t= 1,\dots,T$}{
Receive weight $w_t$ and $x_t$.\;
\If{$w_t > w_{\max}$}{
    Reset $\alg$.\;
    $w_{\max}\leftarrow 2w_t$.
}
Predict $\hat y_t$, where $\hat{y}_t$ is the prediction from $\alg$
given $x_t$.\;
Observe $a_t$ and $\ell_t$.\;
\If{ $u_t\sim \ber(w_t/w_{\max})=1$ }{
    Update $\alg$ with $(x_t, a_t,\ell_t)$.
}
}
\end{algorithm}
In this section we show how to transform any unweighted online
regression oracle $\alg$ satisfying \pref{ass:regression} into a
weighted regression oracle satisfying \pref{ass:weighted
  regression}. The reduction is given in \pref{alg:oracle}, and the performance guarantee is as follows.
\begin{theorem}
If the oracle $\alg$ satisfies \cref{ass:regression} with regret bound
$\RegSquare$, \cref{alg:oracle} satisfies \cref{ass:weighted
  regression} with the same regret bound.
\end{theorem}
\begin{proof}
Let $D_t = (w_t, x_t, a_t, \ell_t)$ and define a filtration
$\gfilt_t = \sigma\prn{D_{1:t}}$, with the convention
$\En_t\brk*{\cdot}=\En\brk*{\cdot\mid\gfilt_t}$. Let $\tau_1,\tau_2\dots, \tau_I$ denote the timesteps at which the
algorithm doubles $w_{\max}$ and resets $\alg$, with the convention that
for all $n>I$, $\tau_{n}=T+1$. Note that these random
variables are stopping times with respect to the filtration
$\gfilt_{1:T}$, and hence $\gfilt_{\tau_i}$ is well-defined for each
$i\in\bbN$. It will be helpful to note that we have
$\tau_{i+1}>\tau_i$ for all $i\leq I$ by construction, and otherwise $\tau_{i+1}=\tau_{i}$. We also observe that $\tau_1=1$ unless $w_1=0$.

For the first step, we show that the conditional regret of
\pref{alg:oracle} between any pair of doubling steps is bounded. Let
$i\leq I$ and $f\in\cF$ be fixed, and observe that $i\leq I$ holds iff $\tau_i\leq T$, which is $\gfilt_{\tau_i}$-measurable. Hence,
\begin{align*}
    &\E\brk*{\sum_{t=\tau_i}^{\tau_{i+1}-1}w_t\prn*{(\ip{a_t,\hat y_t}-\ell_t)^2- (\ip{a_t,f(x_t)}-\ell_t)^2}  \,\mid\, \gfilt_{\tau_i} } \\
&=\E\brk*{2w_{\tau_i}\sum_{t=\tau_i}^{\tau_{i+1}-1}\frac{w_t}{2w_{\tau_i}}\prn*{(\ip{a_t,\hat y_t}-\ell_t)^2- (\ip{a_t,f(x_t)}-\ell_t)^2}  \,\mid\, \gfilt_{\tau_i} }\\
    &\topnote{=}{follows from the conditional independence of $u_t$,
    }\E\brk*{2w_{\tau_i}\sum_{t=\tau_i}^{\tau_{i+1}-1}\E_t\brk*{u_t\prn*{(\ip{a_t,\hat y_t}-\ell_t)^2- (\ip{a_t,f(x_t)}-\ell_t)^2}}  \,\mid\, \gfilt_{\tau_i} }\\
    &\topnote{=}{is by the tower rule of expectation, and
    }\E\brk*{2w_{\tau_i}\sum_{t=\tau_i}^{\tau_{i+1}-1}u_t\prn*{(\ip{a_t,\hat y_t}-\ell_t)^2- (\ip{a_t,f(x_t)}-\ell_t)^2}  \,\mid\, \gfilt_{\tau_i} }\\
    &\topnote{\leq}{uses \cref{ass:regression} on the set $\{t\in\{\tau_i,\dots\tau_{i+1}-1\}\,\mid\, u_t=1\}$ (in particular, that regret is bounded by $\RegSquare$ on every sequence with probability $1$ and $\RegSquare$ is non-decreasing in $T$).
    }
    \E\brk*{2w_{\tau_i}  \,\mid\, \gfilt_{\tau_i} }\cdot\RegSquare\,,
\end{align*}
where \writeeqnotes
For $i>I$, the term is $0$ since the sum is empty.
To complete the proof that \cref{alg:oracle} satisfies \cref{ass:weighted
  regression}, we sum the bound above across all epochs as follows:

\begin{align*}
    &\E\brk*{\sum_{t=1}^{T}w_t\prn*{(\ip{a_t,\hat y_t}-\ell_t)^2- (\ip{a_t,f(x_t)}-\ell_t)^2}}\\
    &\topnote{=}{%
    uses that all $t<\tau_1$ have $w_t=0$, 
    }\E\brk*{\sum_{i=1}^\infty\sum_{t=\tau_i}^{\tau_{i+1}-1}w_t\prn*{(\ip{a_t,\hat y_t}-\ell_t)^2- (\ip{a_t,f(x_t)}-\ell_t)^2}}\\
    &\topnote{=}{%
    uses the tower rule of expectation,
    }\E\brk*{\sum_{i=1}^\infty\E\brk*{\sum_{t=\tau_i}^{\tau_{i+1}-1}w_t\prn*{(\ip{a_t,\hat y_t}-\ell_t)^2- (\ip{a_t,f(x_t)}-\ell_t)^2}\,\mid\,\gfilt_{\tau_i}}}\\
    &\topnote{\leq}{%
    applies the conditional bound between stopping times above,
    }\E\brk*{\sum_{i=1}^I\E\brk{2w_{\tau_i}\,\mid\,\gfilt_{\tau_i}}}\RegSquare\\
    &\topnote{=}{uses the tower rule of expectation again, }2\E\brk*{\sum_{i=0}^Iw_{\tau_i}}\RegSquare\\
    &\topnote{\leq}{%
    holds because the weights at least double between doubling steps, and
    } 2\E\brk*{2w_{\tau_I}}\RegSquare
    \topnote{\leq}{%
    follows because $\tau_I$ is a random variable with support over $[T]$.
    }
    4\E\brk*{\max_{t\in[T]}w_{t}}\RegSquare\,,
\end{align*}
where \writeeqnotes
\end{proof}
\resetnotes
\section{Proofs from Section \ref*{sec:main}}
In this section we provide complete proofs for all of the algorithmic
results from \cref{sec:main}.
\subsection{Proofs from Section \ref*{sec:k-armed}}

\begin{proof}[Proof of \cref{lem:logbar opt}]
We begin by showing that the $\logbarrier(\thetahat,\gamma)$ distribution takes the form
claimed in \pref{eq:logbarrier_form}.
The minimization problem of \cref{lem:logbar opt} is strictly convex and the value approaches $\infty$ at the boundary. Hence the unique solution lies in the interior of the domain.
By the K.K.T. conditions, the partial derivatives for each coordinate must coincide for the minimizer $\pstar$. That is, there exists $\tilde \lambda \in \R$ such that
\begin{align*}
    \forall a\in\brk{K}:\,\frac{\partial}{\partial p_a}\left(\ip{\pstar,\thetahat}-\frac{1}{\gamma}\sum_{a\in\brk{K}}\log(\pstar_a)\right)=\thetahat_a -\frac{1}{\gamma \pstar_a} = \tilde\lambda\,.
\end{align*}
Substituting $\tilde\lambda = \min_{a\in\brk{K}}\thetahat_a-\lambda/\gamma$ and rearranging finishes the proof.

We next show that the $\logbarrier(\thetahat,\gamma)$ distribution indeed solves the minimax problem \pref{eq:minimax}, which we rewrite as
\begin{align}
&\min_{p\in\Delta(\brk{K})}\sup_{\theta\in\mathbb{R}^K}\max_{i^*\in\brk{K}}\ip{\bar a_p-\mathbf{e}_{i^*},\theta}-\frac{\gamma}{4}\norm{\hat\theta-\theta}^2_{H_p}\notag\\
&=\min_{p\in\Delta(\brk{K})}\max_{i^*\in\brk{K}}\sup_{\delta\in\mathbb{R}^K}\ip{\bar a_p-\mathbf{e}_{i^*},\hat\theta+\delta}-\frac{\gamma}{4}\norm{\delta}^2_{H_p}\,.\label{eq:logbarrier_alt}
\end{align}
For any fixed $p$ and $\istar$, the derivative of the
expression in \pref{eq:logbarrier_alt} with respect to $\delta$ is given by
\begin{align}
    \frac{\partial}{\partial \delta}\left[\ip{\bar a_p-\mathbf{e}_{i^*},\delta}-\frac{\gamma}{4}\norm{\delta}^2_{H_p}\right]=\bar a_p-\mathbf{e}_{i^*}-\frac{\gamma}{2}H_p\delta\,.\label{eq:logbarrier_alt_subgradient}
\end{align}
For $p$ on the boundary of $\Delta(\brk{K})$ (i.e. $p$ for which there exists
$i\in\brk{K}$ such that $p_i=0$), the gradient is constant and the
supremum has value $+\infty$. Hence, we only need to consider the case where $p$ lies in the interior of $\Delta(\brk{K})$, which implies $H_p\psdgt{}0$. 
In this case, \pref{eq:logbarrier_alt_subgradient} is strongly convex in $\delta$ and the unique maximizer is
given by $\delta^* = \frac{2}{\gamma}H_p^{-1}(\bar
a_p-\mathbf{e}_{i^*})$. Hence, we can rewrite
\eqref{eq:logbarrier_alt} as
\begin{align}
    &\min_{p\in\Delta(\brk{K})}\max_{i^*\in\brk{K}}\max_{\delta\in\mathbb{R}^K}\ip{\bar a_p-\mathbf{e}_{i^*},\hat\theta+\delta}-\frac{\gamma}{4}\norm{\delta}^2_{H_p}\notag\\
    &=\min_{\substack{p\in\Delta(\brk{K}) \\ H_p\succ 0}}\max_{i^*\in\brk{K}}\ip{\bar a_p-\mathbf{e}_{i^*},\hat\theta}+\frac{1}{\gamma}\norm{\bar a_p-\mathbf{e}_{i^*}}^2_{H_p^{-1}}\notag\\
    &\geq\min_{\substack{p\in\Delta(\brk{K}) \\ H_p\succ 0}}\E_{\istar\sim p}\left[\ip{\bar a_p-\mathbf{e}_{i^*},\hat\theta}+\frac{1}{\gamma}\norm{\bar a_p-\mathbf{e}_{i^*}}^2_{H_p^{-1}}\right]\label{eq:logbarrier_inequality}\\
    &=\min_{\substack{p\in\Delta(\brk{K}) \\ H_p\succ 0}}\E_{\istar\sim p}\left[\frac{1}{\gamma}\left(\tr(H_pH^{-1}_p)-\norm{\bar a_p}^2_{H_p^{-1}}\right)\right]= \frac{K-1}{\gamma}\,.\notag
\end{align}
Now consider the inequality \eqref{eq:logbarrier_inequality}.
If we can show that there exists a unique solution $p$ such that this
step in fact holds with equality, then we have identified the minimizer over
$p\in\Delta(\brk{K})$. Consider an arbitrary candidate solution $p$ on the interior of $\Delta(\brk{K})$.
Then, letting $W_i\ldef \ip{\bar
  a_p-\mathbf{e}_{i^*},\hat\theta}+\frac{1}{\gamma}\norm{\bar
  a_p-\mathbf{e}_{i^*}}^2_{H_p^{-1}}$, the step
\eqref{eq:logbarrier_inequality} lower bounds $\max_{i\in\brk{K}}W_i$
by $\E_{i\sim p}\brk{W_i}$. This step holds with equality if and only if $\E_{i\sim p}\brk{W_i-\max_{i'\in\brk{K}}W_{i'}}=0$.
Since all probabilities $p_i$ are strictly positive, this can happen
if and only if
\begin{align*}
    \exists \tilde\lambda\in\mathbb{R}\quad\text{such that}\quad\forall i\in\brk{K}:\,W_i=\ip{\bar a_p-\mathbf{e}_{i},\hat\theta}+\frac{1}{\gamma}\norm{\bar a_p-\mathbf{e}_{i}}^2_{H_p^{-1}} = \tilde\lambda\,.
\end{align*}
Basic algebra shows that 
\begin{align*}
    \ip{\bar a_p-\mathbf{e}_{i},\hat\theta}+\frac{1}{\gamma}\norm{\bar a_p-\mathbf{e}_{i}}^2_{H_p^{-1}}=\sum_{i'\in[K]}p_{i'}\hat\theta_{i'}-\hat\theta_i -\frac{1}{\gamma}+\frac{1}{\gamma p_i} = \tilde\lambda\,.
\end{align*}
Substituting $\tilde \lambda = \sum_{i'\in[K]}p_{i'}\hat\theta_{i'}-\min_j \hat\theta_j
-\frac{1}{\gamma}+\lambda/\gamma$, rearranging, and picking the unique value for $\lambda$ such that the result is a probability distribution leads to the
precisely the $\logbarrier(\thetahat,\gamma)$ distribution. 
\end{proof}

\subsection{Proofs from Section \ref*{sec:linear}}
Recall that $\dim(\cA)$ is the dimension of the smallest affine linear
subspace containing $\cA$. In other words, $\dim(\cA)=
\dim(\spann(\cA-a))$ for all $a\in\cA$. Our main result in this section is the following slightly stronger version of \cref{lem:logdet bound}.
\begin{lemma}
\label{lem:logdet strong}
Any solution $p\in\Delta(\cA)$ to the problem $\logdetbarrier(\thetahat,\gamma;\cA)$
in \pref{eq:logdet_barrier} satisfies
\begin{align*}
    \max_{a^*\in\cA}\sup_{\theta\in\R^d}\ip{\bar a_p-a^*,\theta}-\frac{\gamma}{4}\norm{\hat\theta-\theta}^2_{H_p-\bar a_p\bar a_p^{\trn}}\leq \gamma^{-1}\dim(\cA)\,.
\end{align*}
\end{lemma}
Since $-\norm{\hat\theta-\theta}^2_{H_p-\bar a_p\bar a_p^{\trn}}=-\norm{\hat\theta-\theta}^2_{H_p}+\ip{\hat\theta-\theta,\bar a_p}^2\geq -\norm{\hat\theta-\theta}^2_{H_p}$, \cref{lem:logdet bound} is a direct corollary of \cref{lem:logdet strong}.

\begin{proof}[Proof of \cref{lem:logdet strong}]
We begin by handling the generate case in which $\dim(\cA)<d$.
\paragraph{Case: $\dim(\cA)<d$.}
We first show that if $\dim(\cA)<d$, there exists a bijection from $\cA$
to a set $\tilde \cA\subset \R^{\dim(\cA)}$ and a projection $P$ taking
the loss estimator $\theta$ into $\R^{\dim(\cA)}$, such that
$\logdetbarrier(\theta,\gamma;\cA)$ and
$\logdetbarrier(P(\theta),\gamma;\tilde \cA)$ are (up to the
bijection) identical, and such that the objective in \cref{lem:logdet strong} coincides for $(\theta,\gamma,\cA)$ and $(P(\theta),\gamma,\tilde{\cA})$. This implies for all subsequent arguments, we can assume without loss of generality that
$\dim(\cA)=d$, since if this does not hold we can work in the subspace outlined in this section.

Pick an arbitrary anchor $\mfa\in\cA$, and let $P$ be the projection
onto $\spann(\cA-\mfa)$, represented with a arbitrary fixed orthonormal
basis for $\spann(\cA-\mfa)$. Let $\tilde \cA = P(\cA-\mfa)$, and for each
$p\in\Delta(\cA)$, let $\tilde p\in\Delta(\tilde \cA)$ be such that
$\tilde p_{P(a-\mfa)}=p_a$ (recall that we define $\Delta(\cA)$ to
have countable support). Observe that for all $\hat\theta\in\R^d$, we have
\begin{align*}
\ip{\bar a_p,\hat\theta}=\E_{a\sim p}\brk*{\ip{P(a-\mfa),P(\hat\theta)}}+\ip{\mfa,\hat\theta}=\ip{\bar a_{\tilde p},P(\hat\theta)}+\ip{\mfa,\hat\theta}\,.
\end{align*}
Recall that we define the determinant function $\det(\cdot)$ in $\logdetbarrier$ as the product over the first $\dim(\cA)$ eigenvalues of $H_p-\bar a_p\bar a_p^{\trn}$. Let $(\nu_i)_{i=1}^{\dim(\cA)}$ denote the corresponding eigenvectors (note that this requires $\nu_i\in\spann(\cA-\mfa)$). We have
\begin{align*}
    &\log\det(H_p-\bar a_p\bar a_p^{\trn})=\sum_{i=1}^{\dim(\cA)}\log(\norm{\nu_i}^2_{H_p-\bar a_p\bar a_p^{\trn}})=
    \sum_{i=1}^{\dim(\cA)}\log(\E_{a\sim p}\brk{\ip{\nu_i,a-\bar a_p}^2} )\\
    &=\sum_{i=1}^{\dim(\cA)}\log(\E_{a\sim p}\brk{\ip{\nu_i,a-\mfa-\E_{a'\sim p}(a'-\mfa)}^2} )=\sum_{i=1}^{\dim(\cA)}\log(\E_{a\sim p}\brk{\ip{P(\nu_i),P(a-\mfa)-\bar a_{\tilde p}}^2} )\\
    &=\sum_{i=1}^{\dim(\cA)}\log(\norm{P(\nu_i)}^2_{H_{\tilde p}-\bar a_{\tilde p}\bar a_{\tilde p}^{\trn}})=\log\det(H_{\tilde p}-\bar a_{\tilde p}\bar a_{\tilde p}^{\trn})\,,
\end{align*}
where we have used the fact that $P$ only changes the representation on $\spann(\cA-\mfa)$ and does not change the identity of the eigenvalues.
Combining these two results immediately shows that for any $p\in\logdetbarrier(\hat\theta,\gamma;\cA)$, we have $\tilde p \in \logdetbarrier(P(\hat\theta),\gamma;\tilde\cA)$ and vice versa.

For the objective in \cref{lem:logdet strong}, we note that
\begin{align*}
\ip{\bar a_p-a^*,\theta}=\ip{ \E_{a\sim p}\brk{P(a-\mfa)}-P(a^*-\mfa),P(\theta)}=\ip{ \bar a_{\tilde p}-(P(a^*-\mfa)),P(\theta)}.
\end{align*}
For the quadratic term, following the same steps as above for $\nu_i$,
we have
\begin{align*}
    \norm{\hat\theta-\theta}^2_{H_p-\bar a_p\bar a_p^{\trn}} = \norm{P(\hat\theta)-P(\theta)}^2_{H_{\tilde p}-\bar a_{\tilde p}\bar a_{\tilde p}^{\trn}}\,.
\end{align*}
and
\begin{align*}
    \ip{\bar a_p-a^*,\theta}-\frac{\gamma}{4}\norm{\hat\theta-\theta}^2_{H_p-\bar a_p\bar a_p^{\trn}} = 
    \ip{\bar a_{\tilde p}-P( a^*-\mfa),P(\theta)}-\frac{\gamma}{4}\norm{P(\hat\theta)-P(\theta)}^2_{H_{\tilde p}-\bar a_{\tilde p}\bar a_{\tilde p}^{\trn}}\,.
\end{align*}
Hence, we have
\begin{align*}
    \max_{a^*\in\cA}\sup_{\theta\in\R^d}\ip{\bar a_p-a^*,\theta}-\frac{\gamma}{4}\norm{\hat\theta-\theta}^2_{H_p-\bar a_p\bar a_p^{\trn}} = \max_{\tilde a^*\in\tilde \cA}\sup_{\tilde \theta\in\R^{\dim(\cA)}}\ip{\bar a_{\tilde p}-\tilde a^*,\tilde \theta}-\frac{\gamma}{4}\norm{P(\hat\theta)-\tilde \theta}^2_{H_{\tilde p}-\bar a_{\tilde p}\bar a_{\tilde p}^{\trn}}\,.
\end{align*}
\paragraph{Case: $\dim(\cA)=d$.}
We now handle the full-dimensional case. Our technical result here is as follows.
\begin{lemma}
\label{lem:logdet bound1}
When $\dim(\cA)=d$, any solution $p\in\Delta(\cA)$ to the problem $\logdetbarrier(\theta,\gamma;\cA)$
in \pref{eq:logdet_barrier} satisfies
 \begin{align*}
 \forall a\in\cA:\,\ip{\bar a_p-a,\theta} + \frac{1}{\gamma}\norm{\bar a_p-a}^2_{H_p^{-1}-\bar a_p\bar a_p^{\trn}} \leq \frac{\dim(\cA)}{\gamma}.
 \end{align*}
\end{lemma}
\begin{proof}
We first observe that any solution $p\in\Delta(\cA)$ to the problem
$\logdetbarrier(\thetahat,\gamma;\cA)$ must be positive definite in the
sense that $H_p-\bar a_p\bar a_p^{\trn}\succ 0$, since otherwise the
objective has value $\infty$; note that $\dim(\cA)=d$ implies that a distribution $p$ with $H_p-\bar a_p\bar a_p^{\trn}\succ 0$ indeed exists. Hence, going forward, we only consider $p$ for which $H_p-\bar a_p\bar
a^{\trn}_p\succ 0$.

Recall $p=\logdetbarrier(\thetahat,\gamma;\cA)$ is any solution to
\begin{align*}
    &\argmin_{p\in\Delta(\cA)}  \crl*{\ip{\bar a_p,\thetahat} -\gamma^{-1} \log\det (H_p-\bar a_p\bar a^{\trn}_p)}\,,
\end{align*}
where  $\Delta(\cA)$ is the set of distributions over countable subsets of $\cA$. Hence we can write
\begin{align*}
    \Delta(\cA) = \crl*{\sum_{i=1}^\infty w_i\mathbf{e}_{A_i} \,|\,w\in\R_+^\bbN,
    A\in\cA^{\bbN},\sum_{i=1}^\infty w_i=1}\,, 
\end{align*}
where $\mathbf{e}_a$ denotes the distribution that selects $a$ with probability $1$. 
By first-order optimality, $p$ is a solution to \pref{eq:logdet_barrier} if and only if
\begin{align*}
    \forall p'\in\Delta(\cA)\colon \sum_{a\in\supp(p)\cup\supp(p')}(p'_a-p_a)\frac{\partial}{\partial p_a}\left[\ip{\bar a_p,\hat \theta} - \frac{1}{\gamma}\log\det(H_p-\bar a_p\bar a_p^{\trn})\right]\geq 0\,.
\end{align*}
By the K.K.T. conditions, this holds if and only if there exists some $\tilde\lambda\in\R$ such that 
\begin{align}
    \forall a\in\supp(p):\frac{\partial}{\partial p_a}\left[\ip{\bar a_p,\hat \theta} - \frac{1}{\gamma}\log\det(H_p-\bar a_p\bar a_p^{\trn})\right]&=\tilde\lambda\label{eq:kkt1}\\
    \forall a\in\cA:\frac{\partial}{\partial p_a}\left[\ip{\bar a_p,\hat \theta} - \frac{1}{\gamma}\log\det(H_p-\bar a_p\bar a_p^{\trn})\right]&\geq\tilde\lambda\,.\label{eq:kkt2}
\end{align}
To find $\tilde\lambda$, we calculate the partial derivative for each action $a$ using the chain rule:
\begin{align*}
&\frac{\partial}{\partial p_a}\left[\ip{\bar a_p,\hat \theta} - \frac{1}{\gamma}\log\det(H_p-\bar a_p\bar a_p^{\trn})\right]\\
&=\ip{a,\hat \theta} - \frac{\det(H_p-\bar a_p\bar a_p^{\trn})\tr((H_p-\bar a_p\bar a_p^{\trn})^{-1} (aa^{\trn} - \bar a_pa^{\trn} - a\bar a_p^{\trn}) )}{\gamma\det(H_p-\bar a_p\bar a_p^{\trn})}\\
&= \ip{a-\bar a_p ,\hat \theta} - \frac{1}{\gamma}\norm{a-\bar a_p}^2_{(H_p-\bar a_p\bar a_p^{\trn})^{-1}}+\frac{1}{\gamma}\norm{\bar a_p}^2_{(H_p-\bar a_p\bar a_p^{\trn})^{-1}}+\ip{\bar a_p,\hat\theta}\,. 
\end{align*}
Using \pref{eq:kkt1} and taking the expectation over $p$ yields
\begin{align*}
    \tilde\lambda &= \E_{a\sim p}\brk*{\frac{\partial}{\partial p_a}\left[\ip{\bar a_p,\hat \theta} - \frac{1}{\gamma}\log\det(H_p-\bar a_p\bar a_p^{\trn})\right]}=
     -\frac{d}{\gamma}+\frac{1}{\gamma}\norm{\bar a_p}^2_{(H_p-\bar a_p\bar a_p^{\trn})^{-1}}+\ip{\bar a_p,\hat\theta}\,.
\end{align*}
Plugging this expression into \pref{eq:kkt2}, we deduce that
\begin{align*}
    \forall a\in\cA: \ip{a-\bar a_p ,\hat \theta} - \frac{1}{\gamma}\norm{a-\bar a_p}^2_{(H_p-\bar a_p\bar a_p^{\trn})^{-1}}\geq -\frac{d}{\gamma}\,.
\end{align*}
Rearranging finishes the proof.
\end{proof}

We now conclude the proof of \pref{lem:logdet strong}. Recall that for any solution $p\in\Delta(\cA)$
to the problem $\logdetbarrier(\thetahat,\gamma;\cA)$, the matrix $H_p-\bar a_p\bar a_p^{\trn}$ is positive definite.
In this case, for any fixed $a^*\in\cA$, the function
\begin{align*}
\theta\mapsto    \ip{\bar a_p-a^*,\theta} - \frac{\gamma}{4}\norm{\hat\theta-\theta}^2_{H_p-\bar a_p\bar a_p^{\trn}}
\end{align*}
is strictly concave in $\theta$, and the maximizer $\theta^*$ may be found by setting the derivative with respect to $\theta$ to $0$. In particular,
\begin{align*}
&\frac{\partial}{\partial \theta}\brk*{\ip{\bar a_p-a^*,\theta} - \frac{\gamma}{4}\norm{\hat\theta-\theta}^2_{H_p-\bar a_p\bar a_p^{\trn}}} =\bar a_p-a^*+\frac{\gamma}{2}(H_p-\bar a_p\bar a_p^T)(\hat\theta-\theta),
\intertext{so that the maximizer is given by}
    &\theta^* = \hat\theta + \frac{2}{\gamma}(H_p-\bar a_p\bar a_p^{\trn})^{-1}(a_p-a^*)\,.
\end{align*}
Substituting in this choice, we have that
\begin{align*}
    \max_{a^*\in\cA}\sup_{\theta\in\mathbb{R}^d}\ip{\bar a_p-a^*,\theta} - \frac{\gamma}{4}\norm{\hat\theta-\theta}^2_{H_p-\bar a_p\bar a_p^{\trn}}
    = \max_{a^*\in\cA}\ip{\bar a_p-a^*,\hat \theta} + \frac{1}{\gamma}\norm{\bar a_p-a}^2_{(H_p-\bar a_p\bar a_p^{\trn})^{-1}} \,.
\end{align*}
The result is obtained by applying \cref{lem:logdet bound1} to the
right-hand side above.
\end{proof}

\subsection{Proofs from Section \ref*{sec:adaptive}}

\begin{proof}[Proof of \cref{thm:impreg}]
  \resetnotes
  Let $m$ be fixed. To keep notation compact, we abbreviate
  $q_t\equiv{}q_{t,h}$, $\rho_{t}\equiv\rho_{t,m}$,
  $\gamma_{t}\equiv\gamma_{t,m}$, $Z_t\equiv{}Z_{t,m}$, and so forth. Consider a fixed sequence $S$, and let $f^*$ be any predictor achieving the value
of $\veps_T(S)$ (if the infimum is not achieved, we can consider a
limit sequence; we omit the details). Recall that since we assume an
oblivious adversary, $\fstar$ is fully determined before the
interaction protocol begins. Finally, let us abbreviate $\theta^*_t =
f^*(x_t)$, $a^*_t = \pi_{f^*}(x_t)$, and
$\pi_t^*(x_t)=\argmin_{a\in\cA_t}\mu(a,x_t)$, where ties are broken arbitrarily. Then we can bound
\begin{align*}
\RegImp&= \E\left[\sum_{t=1}^T\frac{Z_t}{q_t}\left(\mu(a_t,x_t)-\mu(\pi_t^*(x_t),x_t)\right)\right]\\
&\leq \E\left[\sum_{t=1}^T\frac{Z_t}{q_t}\left(\ip{a_t-\pi_t^*(x_t), \theta^*_t} + 2\max_{a\in\cA_t}|\mu(a,x_t)-\ip{a,\theta^*_t}|\right)\right]\\
&\topnote{\leq}{follows from the fact that $\E[Z_t]=q_t$ and the Cauchy-Schwarz inequality, together with the definition of $\veps_T$;
}  \E\left[\sum_{t=1}^T\frac{Z_t}{q_t}\left(\ip{a_t-\pi_t^*(x_t), \theta^*_t}\right)\right]+2\veps_TT\\
&\topnote{\leq}{follows from the definition of the policy $\pi_{f^*}$;
}  \E\left[\sum_{t=1}^T\frac{Z_t}{q_t}\ip{a_t-a^*_t, \theta^*_t}\right]+2\veps_T T\\
&\topnote{=}{is due to the fact that, conditioned on $Z_t=1$, we sample $a_t\sim p_t$ with $\E_{a_t\sim p_t}[a_t]=\bar a_{p_t}$;
}  \E\left[\sum_{t=1}^T\frac{Z_t}{q_t}\left(\ip{\bar a_{p_t}-a^*_t, \theta^*_t}-\frac{\gamma_t}{4}\norm{\hat\theta_t-\theta^*}_{H_{p_t}}^2+\frac{\gamma_t}{4}\norm{\hat\theta_t-\theta^*}_{H_{p_t}}^2\right)\right]+2\veps_T T\\
&\topnote{\leq}{uses \cref{lem:logdet bound};
} \E\left[\sum_{t=1}^T\frac{Z_t}{q_t}\left(\frac{\dim(\cA_t)}{\gamma_t}+\frac{\gamma_t}{4}\norm{\hat\theta_t-\theta^*}_{H_{p_t}}^2\right)\right]+2\veps_T T\\
&\topnote{\leq}{uses $\E_{a_t\sim p_t}[a_ta_t^{\trn}]= H_{p_t}$.
} \E\left[\max_{t\in[T]}\gamma_t^{-1}\right]\sum_{t=1}^T\dim(\cA_t)+\E\left[\sum_{t=1}^T\frac{Z_t}{q_t}\frac{\gamma_t}{4}(\ip{a_t,\hat\theta_t}-\ip{a_t,\theta^*_t})^2\right]+2\veps_T T~.
\end{align*}
Here \writeeqnotes
\resetnotes
Continuing with squared error term above, we have
\begin{align*}
    &\E\left[\sum_{t=1}^T\frac{Z_t}{q_t}\gamma_t(\ip{a_t,\hat\theta_t}-\ip{a_t,\theta^*_t})^2\right]&\\
    &=\E\left[\sum_{t=1}^T\frac{Z_t}{q_t}\gamma_t\prn*{(\langle a_t,\hat \theta_t\rangle-\ell_t)^2-(\langle a_t,\theta^*_t\rangle-\ell_t)^2+2(\ell_t-\langle a_t,\theta^*_t\rangle)\langle a_t,\hat \theta_t-\theta^*_t\rangle}\right] \\
    &\topnote{=}{uses that $\ell_t$ is conditionally independent of $Z_t$ and $a_t$.
    }\E\left[\sum_{t=1}^T\frac{Z_t}{q_t}\gamma_t\prn*{(\langle a_t,\hat \theta_t\rangle-\ell_t)^2-(\langle a_t,\theta^*_t\rangle-\ell_t)^2+2(\mu(a_t,x_t)-\langle a_t,\theta^*_t\rangle)\langle a_t,\hat \theta_t-\theta^*_t\rangle}\right],
    \end{align*}
    where \writeeqnotes
    We bound the term involving the difference of squares as
    \begin{align*}
        &\E\left[\sum_{t=1}^T\frac{Z_t}{q_t}\gamma_t((\langle a_t,\hat \theta_t\rangle-\ell_t)^2-(\langle a_t,\theta^*_t\rangle-\ell_t)^2)\right]
        \leq \E\left[\max_{t\in[T]}\frac{\gamma_{t}}{q_{t}}\right]\RegSquare,
    \end{align*}
    by \cref{ass:weighted regression}.\footnote{Note that \cref{ass:weighted regression} holds with bound $\RegSquare$ even if $\SquareAlg$ is run for less than $T$ timesteps, since we could extend the sequence with 0 weight until time $T$.} 
    For the linear term, we apply the sequence of inequalities
    \resetnotes
    \begin{align*}
    &2\E\left[\sum_{t=1}^T\frac{Z_t}{q_t}\gamma_t(\mu(a_t,x_t)-\langle a_t,\theta^*_t\rangle)\langle a_t,\hat \theta_t-\theta^*_t\rangle\right]\\
    &\topnote{\leq}{is by the AM-GM inequality: $2ab \leq 2a^2+\frac{1}{2}b^2$;
    }
    2\E\left[\sum_{t=1}^T\frac{Z_t}{q_t}\gamma_t((\mu(a_t,x_t)-\langle
      a_t,\theta^*_t\rangle)^2+\frac{1}{4}\langle a_t,\hat
      y_t-\theta^*_t\rangle^2\right] \\
      &\leq{}2\E\left[\sum_{t=1}^T\frac{Z_t}{q_t}\gamma_t\max_{a\in\cA_t}((\mu(a,x_t)-\langle a,\theta^*_t\rangle)^2+\frac{1}{4}\langle a_t,\hat y_t-\theta^*_t\rangle^2\right] \\
    &\topnote{\leq}{follows from the fact that $Z_t$ is conditionally independent of $\gamma_t$, and the definition of $\veps_T$.     }
      2\E\brk*{\max_{t\in[T]}\gamma_t}\veps_T^2 T+\frac{1}{2}\E\left[\sum_{t=1}^T\frac{Z_t}{q_t}\gamma_t(\ip{a_t,\hat\theta_t}-\ip{a_t,\theta^*_t})^2\right],
\end{align*}
where \writeeqnotes

Altogether, we have
\begin{align*}
  &\E\left[\sum_{t=1}^T\frac{Z_t}{q_t}\gamma_t(\ip{a_t,\hat\theta_t}-\ip{a_t,\theta^*_t})^2\right]\\
  &\leq
  \E\left[\max_{t\in[T]}\frac{\gamma_{t}}{q_{t}}\right]\RegSquare + 2\E\brk*{\max_{t\in[T]}\gamma_t}\veps_T^2 T+\frac{1}{2}\E\left[\sum_{t=1}^T\frac{Z_t}{q_t}\gamma_t(\ip{a_t,\hat\theta_t}-\ip{a_t,\theta^*_t})^2\right].
\end{align*}
Rearranging yields
\begin{align*}
    \E\left[\sum_{t=1}^T\frac{Z_t}{q_t}\gamma_t(\ip{a_t,\hat\theta_t}-\ip{a_t,\theta^*_t})^2\right]\leq 2\E\left[\max_{t\in[T]}\frac{\gamma_t}{q_t}\right]\RegSquare  + 4\E\brk*{\max_{t\in[T]}\gamma_t}\veps_T^2T\,.
\end{align*}
Combining all of the developments so far, we have
\begin{align}
  \RegImp &\leq \sum_{t=1}^T\E\left[\gamma_t^{-1}\right]\dim(\cA_t) + \frac{1}{2}\E\left[\max_{t\in[T]}\frac{\gamma_t}{q_t}\right]\RegSquare  + \E\brk*{\max_{t\in[T]}\gamma_t}\veps_T^2T+2\veps_T T\,.\label{eq:nice}
\end{align}
The proof is completed by noting that the learning rate
$\gamma_t=\min\set{\frac{\sqrt{d}}{\veps'},\sqrt{dT/(\rho_t\RegSquare)}}$
is non-increasing, but $\gamma_t\rho_t\geq \frac{\gamma_t}{q_t}$ is
non-decreasing. Hence, we can upper bound the expression above by
\begin{align*}
    \RegImp &\leq \E\left[\gamma_T^{-1}\right]dT + \frac{1}{2}\E\left[\gamma_T\rho_T\right]\RegSquare  + \E[\gamma_1]\veps_T^2T+2\veps_T T\\
    &\leq \left(\frac{\veps'}{\sqrt{d}} + \E[\sqrt{\rho_T}]\sqrt{\frac{\RegSquare}{dT}}\right)dT+\frac{1}{2}\E[\sqrt{\rho_T}]\sqrt{dT\RegSquare}+\frac{\sqrt{d}}{\veps'}\veps_T^2T+2\veps_T T\,.
\end{align*}
\end{proof}

\begin{proof}[Proof of \cref{thm:main}]
Let $\mstar \ldef
\argmin_{m\in\brk*{M}}\frac{\veps_T}{\veps'_m}+\frac{\veps'_m}{\veps_T}$
if $\veps_T\geq T^{-1}$ and $\mstar\ldef{}M$ otherwise. 
We begin by formally verifying the claim
\begin{align}
    \Reg = \E\brk*{\sum_{t=1}^T\tilde\ell_{t,A_t}-\tilde\ell_{t,\mstar}}+\RegImp[\mstar]\,.\label{eq:regsplit}
\end{align}
By the definition
$\tilde\ell_{t,A_t} \ldef \ell_t+1$, we have
\begin{align*}
    \E\brk*{\tilde\ell_{t,A_t}-\tilde\ell_{t,\mstar}}=\E\brk*{\ell_t+1-\frac{Z_{t,\mstar}}{p_{t,\mstar}}(\ell_t+1)}=\E\brk*{\mu(a_t,x_t)-\frac{Z_{t,\mstar}}{p_{t,\mstar}}\mu(a_t,x_t)}\,.
\end{align*}
On the other hand, the second term on the right-hand side of \pref{eq:regsplit} is
\begin{align*}
    \RegImp[\mstar]=\E\brk*{\sum_{t=1}^T\frac{Z_{t,\mstar}}{p_{t,\mstar}}\prn*{\mu(a_t,x_t)-\mu(\pi^*_t(x_t),x_t)}}=\E\brk*{\sum_{t=1}^T\frac{Z_{t,\mstar}}{p_{t,\mstar}}\mu(a_t,x_t)-\mu(\pi^*_t(x_t),x_t)}\,.
\end{align*}
Combining both lines leads to the identity in \pref{eq:regsplit}.

Proceeding with the proof, recall that the losses $\tilde\ell$ satisfy $\tilde\ell_{t,m}\in[0,2]$ for all $m$, since $\ell_t\in[-1,1]$ and we shift the loss by $1$. Hence, we can apply \cref{cor:Tsallis-INF} with $\alpha=\frac{1}{2}$ and
$R=\frac{3}{2}\sqrt{dT\RegSquare }$ to obtain
\begin{align*}
    \E\brk*{\sum_{t=1}^T\tilde\ell_{t,A_t}-\tilde\ell_{t,\mstar}} \leq 4\sqrt{2MT} +3\sqrt{dT\RegSquare M}-\frac{3}{2}\E[\sqrt{\rho_{T,a^*}}]\sqrt{dT\RegSquare }\,, 
\end{align*}
and by \cref{thm:impreg},
\begin{align*}
    \RegImp[\mstar]\leq\left(\left(\frac{\veps'_{\mstar}}{\veps_T}+\frac{\veps_T}{\veps'_{\mstar}} \right)\sqrt{d}+2\right)\veps_TT + \frac{3}{2}\E[\sqrt{\rho_{T,a^*}}]\sqrt{dT\RegSquare }\,.
\end{align*}
We now consider two cases. First, if $\veps_T>T^{-1}$, we can pick $\mstar$ such that $\veps'_{\mstar}\in[\veps_T,e\veps_T]$, which ensures that
$\left(\frac{\veps'_{\mstar}}{\veps_T}+\frac{\veps_T}{\veps'_{\mstar}}\right)\leq
e+e^{-1}$. Otherwise, we pick $\veps'_{\mstar} = T^{-1}$ so that the misspecification term is bounded by 
\begin{align*}
    \left(\left(\frac{\veps'_{\mstar}}{\veps_T}+\frac{\veps_T}{\veps'_{\mstar}} \right)\sqrt{d}+2\right)\veps_T T = \left(\left(\veps'_{\mstar}+\frac{\veps_T^2}{\veps'_{\mstar}} \right)\sqrt{d}+2\veps_T\right) T\leq 2\sqrt{d}+2\,.
\end{align*}
Summing the regret bounds for the base and master algorithms completes
the proof.
\end{proof}

\subsection{Proofs from Section \ref*{sec:sparse}}

\paragraph{Algorithm description}
We begin by outlining the algorithm that achieves the bound in \pref{thm:sparse}. The algorithm proceeds in episodes. At the begin of episode 1, the algorithm defines $D_1 \ldef \sum_{t=1}^T\dim(\cA_t)\leq 2T$ and $\tau_1=1$ and plays the algorithm from \pref{thm:main} with the learning rate tuned for $d=D_1$.
Within each episode $i\geq{}1$, if the agent observes at time $t$ that
$\sum_{s=\tau_i}^t\dim(\cA_s)>D_i$, it restarts the algorithm from \pref{thm:main} with
$D_{i+1}\ldef2D_i$; we denote this time by $\tau_{i+1}=t$.
Note that we can assume $\dim(\cA)\leq d <T$ without loss of generality (otherwise the result is trivial), and hence we never need to double more than once at each time step.

To analyze this algorithm, we first show that the bound from \cref{thm:main} continues to hold even if the learner plays only on a subset of time steps.
\begin{proposition}
\label{prop:anytime main thm}
Let $\cT \subset [T]$ be an obliviously chosen subset of timesteps.
Then the upper bound from \cref{thm:main} on a sequence $S$ continues to hold if the algorithm is run on a sub-sequence $S_\cT$.
\end{proposition}
\begin{proof}
We extend the sequence $S$ by adding an ``end'' sequence $E=(\{0\},\mathfrak{x})_{t=1}^T$, where $\mathfrak{x}\in\cX$ is picked such that $\mu(0,\mathfrak{x})=0$ (If there is no such context, we add a context with that property to $\cX$).
Let the extended sequence be $S'=S + E$ and consider the sequence $\tilde S = S'_{\cT\cup\{T+1,\dots,2T-|\cT|\}}$, which has length $T$.
The contribution to regret from playing on the $E$ section on the sequence is always 0, since there is
only one action. 
Furthermore $\veps_{T}(\tilde S)\leq \veps_T(S)$. Hence, by \cref{thm:main}, we have
\begin{align*}
    \E\left[\sum_{t\in\cT}\mu(a_t,x_t)-\min_{a\in\cA_t}\mu(a,x_t)\right]
    &=\E\left[\sum_{t\in\cT\cup\{T+1,\dots,2T-|\cT|\}}\mu(a_t,x_t)-\min_{a\in\cA_t}\mu(a,x_t)\right] \\
    &\leq\cO\prn*{\sqrt{d}\veps_T(S)T+\sqrt{d\RegSquare\log(T)}}\,.
\end{align*}
\end{proof}
Note that since $\dbar(\tilde S)\leq \dbar(S)$, this argument also applies to the refined version of the bound where $d$ is replaced by $\dbar$, as long as the algorithm's parameters are tuned accordingly.

\begin{proof}[Proof of \cref{thm:sparse}]
Let $\tau_1,\dots,\tau_L$ denote the times where the algorithm is
restarted, with $\tau_1=1$ and $\tau_{L+1}=T+1$ by convention. Since the adversary fixes the action sets in advance, these doubling times are deterministic.
The regret is given by
    \begin{align*}
    \textstyle\Reg &= \E\left[\sum_{t=1}^T\mu(a_t,x_t)-\min_{a\in\cA_t}\mu(a,x_t)\right]\\
    &\leq\sum_{i=1}^L\E\left[\sum_{t=\tau_{i}}^{\tau_{i+1}-1}\mu(a_t,x_t)-\min_{a\in\cA_t}\mu(a,x_t)\right]\,.
    \end{align*}
By applying
\cref{prop:anytime main thm} to each episode, we have
\begin{align*}
    \E\left[\sum_{t=\tau_{i-1}+1}^{\tau_i}\mu(a_t,x_t)-\min_{a\in\cA_t}\mu(a,x_t)\right]= \sqrt{2^i}\cdot\cO\prn*{\veps_T(S)T+\sqrt{\RegSquare\log(T)}}\,.
\end{align*}
Summing over these terms and observing that
\begin{align*}
    \sum_{i=1}^L 2^{i/2} = \cO(2^{L/2}) = \cO\prn*{1}\cdot\sqrt{\frac{1}{T}\sum_{t=1}^T\dim(\cA_t)}=\bigoh\prn*{\dbar^{1/2}}
\end{align*}
completes the proof.
\end{proof}

\section{Improved Master Algorithms for Bandit Aggregation}
\label{sec:hedged ftrl}

In this section we present a new family of algorithms that can be
used for the master algorithm within the framework of \pref{alg:master}.
For the remainder of this section, we work in a generic adversarial multi-armed bandit setting, at each time step, the agent selects an action $A_t\in [M]$, then observes a loss $\ell_{t,A_t}\in[0,L]$ for the action they selected. Compared to the log-barrier-based master algorithm used within the original \corral algorithm of \citet{ALNS17}, the algorithms we describe here are simpler to analyze, more flexible, and have improved logarithmic factors.

\subsection{Background and Motivation}
The \corral algorithm is a special case of \cref{alg:master} that
uses a bandit variant of the Online Mirror Descent (OMD) algorithm with
log-barrier regularization as the master.\footnote{Note that the use
  of the log-barrier in \corral is not related to our use of the
  log-barrier within the contextual bandit framework.}
The bandit variant of the OMD algorithm used within \corral is parameterized by a Legendre
potential $F(x) = \sum_{i=1}^d\eta_i^{-1}f(x_i)$ where
$\eta_1,\ldots,\eta_d$ are per-coordinate learning rates. It is initialized using the distribution $p_1 = \argmin_{p\in\Delta([M])}F(p)$. Then, at each time $t$, the bandit OMD algorithm samples an arm $A_t\sim p_t$,
observes the loss $\ell_{t,A_t}$, and constructs an unbiased importance-weighted loss
estimator $\hat\ell_t =
\frac{\ell_{t,A_t}}{p_{t,A_t}}\mathbf{e}_{A_t}$. It then updates the
action distribution via
\begin{align}
\label{eq:OMD update}
    \textstyle p_{t+1} = \argmin_{p\in\Delta([M])} \crl*{\ip{p,\hat\ell} + D_F(p,p_t)}\,,
\end{align}
where $D_F(x,y)\ldef{}F(x)-F(y)-\ip{x-y,\nabla F(y)}$ is the Bregman divergence associated with $F$. An important feature which leads to the guarantee for the \corral master is a time-dependent learning rate schedule for each of the per-arm learning rates, which increases the learning rate for each arm whenever the probability for that arm falls below a certain threshold.\footnote{For time-dependent learning rates, we replace $\eta$ by $\eta_t$ in the update rule of \cref{eq:OMD update}.}

Online Mirror Descent is closely related to the Follow-the-Regularized-Leader (FTRL) algorithm. In particular, for any sequence of loss vector estimates $(\hat\ell_t)_{t=1}^{T}$, there exists a sequence of (vector) biases $b_t$ such that FTRL running on the loss sequence $(\hat\ell_t-b_t)_{t=1}^{T}$ using the same learning rate as its OMD counterpart has an identical trajectory of plays $p_t$. We can view the \corral master through the lens of FTRL. In particular, the FTRL variant of the algorithm performs two steps whenever it increases the learning rate of arm $i$. First it subtracts a bias $b_{t, i}>0$ from the loss estimates for arm $i$. Then it increases the learning rate for that arm. We show that only the former step is actually required, while the latter is unnecessary. This motivates the \emph{$(\alpha,R)$-hedged FTRL}
algorithm, which achieves a slightly improved guarantee by removing
the per-coordinate learning rates.
\subsection{The Hedged FTRL Algorithm}
Following the intuition in the prequel, we present $(\alpha, R)$-hedged FTRL, a modified variant of the FTRL
algorithm with strong guarantees for aggregating bandit
algorithms. To do so, we first describe a basic bandit variant of FTRL algorithm.

The FTRL family of algorithms is parameterized by a potential $F$ and
learning rate $\eta > 0$. At each round $t$, the algorithm selects 
\begin{align*}
    \textstyle p_t = \argmin_{p\in\Delta([M])}\crl*{\ip{p,\hat L_{t-1}} + \eta^{-1}F(p)}\,,\quad\text{where}\quad\hat L_{t}=\sum_{s=1}^t\hat\ell_s\,.
\end{align*}
Two relevant properties of $F$ that arise in our analysis are \emph{stability} and \emph{diameter}.
Define
\begin{align*}
    \textstyle\bar F_\eta^*(-L) = \max_{p\in\Delta([M])}\crl*{\ip{p,- L} - \eta^{-1}F(p)}\,.
\end{align*}
The stability $\stab(F)$ and diameter $\diam(F)$ of $F$ for loss range $[0,L]$ are defined as follows:
\begin{align*}
    & \stab(F) =\sup_{\eta>0} \sup_{x\in\Delta([M])}\sup_{\ell\in[0,L]^M} \eta^{-1}\E_{A\sim x}\left[ D_{\bar F_\eta^*}\prn*{\eta^{-1}\nabla F(x)-\frac{\ell_{A}}{x_{A}}\mathbf{e}_{A},\eta^{-1}\nabla F(x)}\right]\,,\\
    & \diam(F) = \max_{p\in\Delta([M])} F(p)-\min_{p\in\Delta([M])} F(p)\,.
\end{align*}
Given a potential with bounded $\stab(F)$ and $\diam(F)$, setting
the learning rate as $\eta = \sqrt{\diam(F)/(\stab(F)T)}$ leads
to regret at most $2\sqrt{\stab(F)\diam(F)T}$ \citep{ALT15}.\footnote{\citet{ALT15} present this result slightly
  differently. See our proof of \cref{thm:hedged ftrl} with
  $R=0$ for an alternative.} Well-known algorithms that arise as
special cases of this result include:
\begin{itemize}
\item EXP3 \citep{ACFS02} is an instantiation of bandit FTRL with
  $F(x) = \sum_{i=1}^M x_i\log(x_i)$, $\diam(F) = \log(M)$ and
  $\stab(F) \leq \frac{L^2M}{2}$.
  \item Tsallis-INF
  \citep{AB09,ALT15,ZiSe18} is an instantiation of bandit FTRL that gives the best
  known regret bound for multi-armed bandits. It is given by $F(x)=-2\sum_{i=1}^M\sqrt{x_i}$,
  which has $\diam(F)\leq 2\sqrt{M}$ and $\stab(F)\leq L^2\sqrt{M}$.
\end{itemize}

We now present the $(\alpha, R)$-hedged FTRL algorithm. The
algorithm augments the basic bandit FTRL strategy using an additional pair of
parameters $\alpha\in(0,1)$, $R\in\R$. The algorithm begins by initializing a collection of parameters $(B_{0,i})_{i=1}^M$ with $B_{0,i}=\rho_{1,i}^\alpha R$.
At each step $t$, it plays $A_t\sim p_t$, then computes
\begin{align*}
    \textstyle\tilde p_{t+1} = \argmin_{p\in\Delta([M])}\ip{p,\hat L_{t}-(B_{t-1}-B_0)} + \eta^{-1}F(p)\,,\quad\text{where}\quad\hat L_{t}=\sum_{s=1}^t\hat\ell_s\,.
\end{align*}
If $\tilde p_{t+1, A_t}^{-\alpha}R \leq B_{t-1,A_t}$, the algorithm sets $B_t=B_{t-1}$ and $p_{t+1}=\tilde p_{t+1}$.
Otherwise it chooses the unique $b_t>0$, such that for $B_t = B_{t-1}+b_t\mathbf{e}_{A_t}$, the following properties hold simultaneously:
\begin{align*}
p_{t+1} = \argmin_{p\in\Delta([M])}\ip{p,\hat L_{t}-(B_{t}-B_0)} + \eta^{-1}F(p)\qquad
\mbox{and}
\qquad
p_{t+1,A_t}^{-\alpha}R = B_{t,A_t}\,.
\end{align*}

This algorithm is always well defined when the potential $F$ is symmetric; see \ifsup \cref{app:hedged
  ftrl} \else the supplementary material \fi for details. Letting $\rho_{t,i} = \max_{s\leq t}p_{s,i}^{-1}$, the main
regret guarantee is as follows. 
\begin{theorem}
\label{thm:hedged ftrl}
For any potential $F$ with $\stab(F), \diam(F) < \infty$, the pseudo-regret $\Reg[M]=\E\left[\sum_{t=1}^T\ell_{t,A_t}-\ell_{t,a^*}\right]$ for $(\alpha,R)$-hedged FTRL with learning rate $\eta =\sqrt{\diam(F)/(\stab(F)T)}$ is bounded by
\begin{equation*}
    \Reg[M]\leq 2\sqrt{\stab(F)\diam(F)T}+ \left[\frac{\alpha}{1-\alpha}\,\sum_{i=1}^M\left( \rho_{1,i}^{\alpha-1} -  \E[\rho_{T,i}^{\alpha-1}]\right)+ \rho_{1,a^*}^{\alpha}- \E[\rho_{T,a^*}^{\alpha}]\right]\cdot{}R\,
\end{equation*}
for all arms $a^*\in[M]$.
\end{theorem}
This algorithm may be viewed ``hedging'' against the event that the arm
$a^*$ experiences a very small probability, as this leads to a negative regret contribution proportional to $\rho_{T,a^*}^{-\alpha}R$.
\subsection{Proofs}
\label{app:hedged ftrl}
Before proving the main result, we first establish that the
$(\alpha,R)$-hedged FTRL strategy as described is in fact well-defined. Recall that the algorithm initializes with $B_0$ such that $\nabla \bar F^*_\eta(B_0)_i^{-\alpha}R=B_{0,i}$.
For symmetric potentials $F(x)=\sum_{i=1}^Mf(x_i)$, $\nabla \bar F^*_\eta(c\mathbf{1}_M)=\frac{1}{M}\mathbf{1}_M$ for all $c\in\R$. Hence $B_0=M^{-\alpha}R\mathbf{1}_M$ satisfies the initialization condition.
Otherwise a
solution exists by the observation that $\nabla \bar
F^*_\eta(B_0)_i^{-\alpha}R$ is a continuous, decreasing function in
$B_{0,i}$ that has positive values at $B_{0}=0$. Hence a solution to the equation must
exist. 

The same argument holds during the update at subsequent rounds $t$. Only the arm that was played can decrease in probability, which means we only need to ensure that $\rho_{t+1,A_t}^{\alpha}R = B_{t,A_t}$. The left-hand side is continuously decreasing with increasing $b_t$, while the right-hand side is increasing.
Hence, the optimal value must exist, is unique, and lies in $[0,\hat\ell_{t,A_t}]$. 

\begin{proof}[Proof of \cref{thm:hedged ftrl}]
We follow the standard FTRL analysis. Let $\tilde B_t=B_t-B_0$ and
note that $p_t = \nabla \bar F^*_\eta(-\hat L_{t-1}+\tilde B_{t-1})$,
so $\ip{p_t,\hat\ell_t} = \ip{\nabla \bar F^*_\eta(-\hat
  L_{t-1}+\tilde B_{t-1}), \hat L_t-\hat L_{t-1}}$. Hence, we can write
\begin{align*}
    &\E\left[\sum_{t=1}^T\ell_{t,A_t}-\ell_{t,a^*}\right] =
    \E\left[\sum_{t=1}^T\ip{p_t,\hat\ell_{t}}-\hat\ell_{t,a^*}\right] \\
    &=\E\left[\sum_{t=1}^T D_{\bar F^*_\eta}(-\hat L_t+\tilde B_{t-1},-\hat L_{t-1}+\tilde B_{t-1})\right]\\
    &\qquad\qquad +\E\left[\sum_{t=1}^T\left(-\bar F^*_\eta(-\hat L_t+\tilde B_{t-1})+\bar F^*_\eta(-\hat L_{t-1}+\tilde B_{t-1})-\hat\ell_{t,a^*}\right)\right].
\end{align*}
Note that there exists $\lambda$ such that $-\hat L_{t-1}+\tilde B_{t-1}=\lambda\mathbf{1}_M + \eta^{-1}\nabla F(p_t)$. Furthermore, adding or subtracting the same $\lambda \mathbf{1}_M$ term to both arguments does not change the value of the Bregman divergence, because $\bar F_\eta(-L+\lambda \mathbf{1}_M) = F_\eta(-L)+\lambda$. Thus,
\begin{align*}
    &\E\left[\sum_{t=1}^T D_{\bar F^*_\eta}(-\hat L_t+\tilde B_{t-1},-\hat L_{t-1}+\tilde B_{t-1})\right] \\
    &=\eta \E\left[\sum_{t=1}^T \eta^{-1}D_{\bar F^*_\eta}(\eta^{-1}\nabla F(p_t)-\hat\ell_t,\eta^{-1}\nabla F(p_t))\right] \leq \eta \stab(F)T\,.
\end{align*}
Rearranging the second term gives
\begin{align*}
    &\sum_{t=1}^T\left(-\bar F^*_\eta(-\hat L_t+\tilde B_{t-1})+\bar F^*_\eta(-\hat L_{t-1}+\tilde B_{t-1})-\hat\ell_{t,a^*}\right)\\
    &=\bar F^*_\eta(0)-\bar F^*_\eta(-\hat L_T+\tilde B_{T-1})-\hat L_{T,a^*} +\sum_{t=1}^{T-1}\bar F^*_\eta(-\hat L_t+\tilde B_t)-\bar F^*_\eta(-\hat L_t+\tilde B_{t-1})\,.
\end{align*}
Note that $\bar F^*_\eta(-\hat L_t+\tilde B_t) = \ip{p_{t+1},-\hat L_t+\tilde B_t} +\eta^{-1}F(p_{t+1})$.
Furthermore we have the bounds
\begin{align*}
    &-\bar F^*_\eta(-\hat L_T+\tilde B_{T-1}) \leq -\left( \ip{\mathbf{e}_{a^*},-\hat L_T+\tilde B_{T-1}}-\eta^{-1}F(\mathbf{e}_{a^*}\right),\intertext{and}
    &-\bar F^*_\eta(-\hat L_t+\tilde B_{t-1}) \leq -\left( \ip{p_{t+1},-\hat L_t+\tilde B_{t-1}}-\eta^{-1}F(p_{t+1}\right).\\
\end{align*}
Plugging these inequalities in above leads to
\begin{align*}
    &\bar F^*_\eta(0)-\bar F^*_\eta(-\hat L_T+\tilde B_{T-1})-\hat L_{T,a^*} +\sum_{t=1}^{T-1}\bar F^*_\eta(-\hat L_t+\tilde B_t)-\bar F^*_\eta(-\hat L_t+\tilde B_{t-1})\\
    &\leq \frac{F(\mathbf{e}_{a^*})-F(p_1)}{\eta}-\tilde B_{T-1,a^*}+\sum_{t=1}^{T-1}\ip{p_{t+1},\tilde B_{t}-\tilde B_{t-1}}\\
    &\leq (\rho_{1,a^*}^\alpha-\rho_{T,a^*}^{\alpha})R +\frac{\diam(F)}{\eta}+\sum_{t=1}^{T-1}\ip{p_{t+1},B_{t}-B_{t-1}}\,.
\end{align*}
To bound the final sum, note that for each coordinate $i$, the difference
$B_{t,i}-B_{t-1,i}$ can be non-zero only if $p_{t+1,i}$ satisfies
$p_{t+1,i}=\rho_{t+1,i}^{-1}$. It follows that
\begin{align*}
    p_{t+1,i}(B_{t,i}-B_{t-1,i}) &= R\rho_{t+1,i}^{-1}\left(\rho_{t+1,i}^\alpha - \rho_{t,i}^\alpha\right)\\
    &= \alpha R\int_{\rho_{t,i}}^{\rho_{t+1,i}}x^{\alpha-1}\rho_{t+1}^{-1}\,dx\\
    &\leq\alpha R\int_{\rho_{t,i}}^{\rho_{t+1,i}}x^{\alpha-2}\,dx\\
    &=\frac{\alpha R}{1-\alpha}(\rho_{t,i}^{\alpha-1}-\rho_{t+1,i}^{\alpha-1})~.
\end{align*}
Applying this bound to each coordinate, we have
\begin{align*}
    \sum_{t=1}^{T-1}\ip{p_{t+1},B_{t}-B_{t-1}} = \sum_{i=1}^M \frac{\alpha R}{1-\alpha} \left(\rho_{1,i}^{\alpha-1}-\rho_{T,i}^{\alpha-1}\right)=\sum_{i=1}^M \frac{\alpha R}{1-\alpha} \left(\rho_{1,i}^{\alpha-1}-\rho_{T,i}^{\alpha-1}\right)\,.
\end{align*}
Combining all of the bounds above concludes the proof.
\end{proof}

\begin{proof}[Proof of \cref{cor:Tsallis-INF}]
Recall that the Tsallis regularizer is given by
\begin{align*}
    F(x) = -\sum_{i=1}^M 2\sqrt{x_i}\,.
\end{align*}
For loss range $[0,L]$, the regularizer has stability at most
$L^2\sqrt{M}$ and diameter at most $\sqrt{M}$ \citep{ZiSe18}.\footnote{\citet{ZiSe18} show this for $L=1$, but the extension
  to general $L$ is trivial.}
Furthermore, since the potential is symmetric, we have $\forall i:\,p_{1,i}= 1/M$. 
Using \cref{thm:hedged ftrl} with the loss range $[0,2]$ leads to
\begin{align*}
    \Reg[M] &\leq 4\sqrt{2MT} + \left[\frac{\alpha}{1-\alpha}\sum_{i=1}^M(M^{\alpha-1}-\E[\rho_{T,i}^{\alpha-1}])+M^\alpha-\E[\rho_{T,m^*}^\alpha]\right]R\\
    &\leq 4\sqrt{2MT} + \left[\frac{\alpha}{1-\alpha}M^\alpha\left(1-M^{1-\alpha}\min_{j\in[M]}\E[\rho_{T,j}^{\alpha-1}]\right)+M^\alpha-\E[\rho_{T,m^*}^\alpha]\right]R\,.
\end{align*}
Dropping the negative
$-M^{1-\alpha}\min_{j\in[M]}\E[\rho_{T,j}^{\alpha-1}]$ term above leads to the
first part of the $\min\crl{\cdot}$ expression in
\pref{eq:tsallis_regret}. For the other term in the $\min\crl{\cdot}$,
note that the function
\begin{align*}
\alpha\mapsto    \frac{\alpha}{1-\alpha}\left(1-z^{\alpha-1}\right)
\end{align*}
is monotonically increasing in $\alpha$, with
\begin{align*}
    \lim_{\alpha\rightarrow 1}\frac{\alpha}{1-\alpha}\left(1-z^{\alpha-1}\right)=\log(z)\,.
\end{align*}
Bounding $\log(\max_{j\in[M]}\E[\rho_{T,j}]/M)+1$ by
$2\log(\max_{j\in[M]}\E[\rho_{T,j}])$ using that $\rho_{1,i}=M$ completes the proof.
\end{proof}

\section{\neurips{Approximation }Algorithms for the Log-Determinant
    Barrier Problem}
\label{app:logdet solver}
Recall that at each step, \infalg (\pref{alg:logdet}) samples from the $\logdetbarrier(\hat\theta,\gamma;\cA)$ distribution, which we define as any (not necessarily unique) distribution in the set
\begin{align}
\label{eq:logdetbarrier}
    \pstar \in \argmin_{p\in\Delta(\cA)}\gamma\ip{\bar a_p,\hat\theta} - \frac{1}{\gamma}\log\det\left(H_p-\bar a_p\bar a_p^{\trn}\right)\,,
\end{align}
where $\bar a_p = \E_{a\sim p}[a]$ and $H_p = \E_{a\sim
  p}[aa^{\trn}]$. In this section, we develop optimization
algorithms to efficiently find approximate solutions to the problem
\pref{eq:logdetbarrier}. Our main result here is to prove \pref{prop:logdet_efficient} as a consequence of a more general result, \pref{thm:fw}.

While \pref{eq:logdetbarrier} is a convex optimization problem,
developing efficient algorithms presents a number of technical
difficulties. First, the optimization problem is non-smooth due to the
presence of the log-determinant function, which prevents us from
applying standard first-order methods such as gradient descent out of
the box. Second, representing distributions in $\Delta(\cA)$ naively
requires $\Omega(\abs*{\cA})$ memory. To get the result in
\pref{prop:logdet_efficient}, we employ a specialized Frank-Wolfe-type method,
which maintains a sparse distribution and requires only
$\bigoh(\log\abs*{\cA})$ memory.

As a first step toward solving \pref{eq:logdetbarrier} numerically, we move to an
equivalent but slightly more convenient formulation which lifts the actions to $d+1$
dimensions. Define the \emph{lifting} operator, which adds a new coordinate with 1 to each
vector, by
\begin{align*}
  \tilde a \ldef \left(
  \begin{array}{l}
    a\\1
  \end{array}\right),
\end{align*}
and define
\begin{align*}
   \tilde a_p \ldef \E_{a\sim p}[\tilde a],\quad
  \tilde H_p \ldef \E_{a\sim p}\brk*{\tilde a\tilde a^{\trn}},\quad
    \tilde \theta \ldef \left(
  \begin{array}{l}
    \thetahat\\0
  \end{array}
\right),
  \quad\text{and}\quad
    \tilde d \ldef d+1\,.
\end{align*}
Finally, define
\begin{align}
  \label{eq:Gp}
G(p) = \ip{\tilde a_p,\tilde\theta} - \frac{1}{\gamma}\log\det(\tilde H_p).
\end{align}
\begin{proposition}
  \label{prop:lifted}
The set of solutions for the lifted problem
\begin{align}
\label{eq:enhanced}
\argmin_{p\in\Delta(\cA)} G(p)=\argmin_{p\in\Delta(\cA)} \ip{\tilde a_p,\tilde\theta} - \frac{1}{\gamma}\log\det(\tilde H_p)\,,
\end{align}
is identical to the set of solutions for  \cref{eq:logdetbarrier}, and vice-versa.
\end{proposition}
\begin{proof}
By \cref{lem:logdet bound1},
 any solution $\pstar$ to
\cref{eq:logdetbarrier} must satisfy the optimality condition
\begin{align*}
    \forall a\in\cA\colon \ip{\bar a_{\pstar}-a,\hat\theta}+\frac{1}{\gamma}\norm{\bar a_{\pstar}-a}^2_{(H_{\pstar}-\bar a_{\pstar}\bar a_{\pstar}^{\trn})^{-1}}\leq \frac{d}{\gamma}\,.
\end{align*}
Now, let $\pstar$ be a minimizer for the optimization problem in
\eqref{eq:enhanced}. 
By first order optimality, we have
\begin{align*}
    \forall p'\in\Delta(\cA)\colon \sum_{a\in\supp(\pstar)\cup\supp(p')}(p'_a-\pstar_a)\prn*{\ip{\tilde a,\tilde \theta} - \frac{1}{\gamma}\norm{\tilde a}^{2}_{\tilde H_{\pstar}^{-1}}}\geq 0\,.
\end{align*}
By the K.K.T. conditions, this condition holds if and only if there exists $\lambda \in \R$ such that
\begin{align}
    \forall a\in\supp(\pstar)\colon &\ip{\tilde a,\tilde \theta} - \frac{1}{\gamma}\norm{\tilde a}^{2}_{\tilde H_{\pstar}^{-1}} = \lambda\label{eq:logdet_opt1}\intertext{and}
    \forall a\in\cA\colon &\ip{\tilde a,\tilde \theta} - \frac{1}{\gamma}\norm{\tilde a}^{2}_{\tilde H_{\pstar}^{-1}} \geq \lambda\,.\label{eq:logdet_opt2}
\end{align}
Note that \pref{eq:logdet_opt1} implies that
\[
  \E_{a\sim\pstar}\brk*{\ip{\tilde a,\tilde \theta} - \frac{1}{\gamma}\norm{\tilde a}^{2}_{\tilde H_{\pstar}^{-1}}}=\ip{\tilde a_{\pstar},\hat\theta}
  -\frac{\tilde d}{\gamma}=\lambda.
\]
Combining this identity with \pref{eq:logdet_opt2} and rearranging, we
conclude that
\begin{align}
    \forall a\in\cA\,:\ip{\tilde a_{\pstar}-a,\hat\theta} +\frac{1}{\gamma}\norm{\tilde a}^{2}_{\tilde H_{\pstar}^{-1}}  \leq \frac{\tilde d}{\gamma}.\label{eq:logdet_opt3}
\end{align}
Finally, observe that for any $p\in\Delta(\cA)$
\begin{align*}
    &\tilde H_{p} = \begin{pmatrix} H_{p} & \bar a_{p}\\\bar a^{\trn}_{p}&1\end{pmatrix},\quad\text{and}\quad
    &\tilde H_{p}^{-1} = \begin{pmatrix}
    \left(H_{p}-\bar a_{p}\bar a_{p}^{\trn}\right)^{-1}&-\left(H_{p}-\bar a_{p}\bar a_{p}^{\trn}\right)^{-1}\bar a_{p}\\
    -\bar a_{p}^{\trn}\left(H_{p}-\bar a_{p}\bar a_{p}^{\trn}\right)^{-1}& 1+\norm{\bar a_{p}}^2_{\left(H_{p}-\bar a_{p}\bar a_{p}^{\trn}\right)^{-1}}
    \end{pmatrix}\,,
\end{align*}
where the second expression uses the identity for the Schur
complement. Using the latter expression, we have that
\begin{align}
    \norm{\tilde a}^{2}_{\tilde H_{p}^{-1}} &=
                                                       \norm{a}^2_{(H_{
                                                       p}-\bar
                                                       a_{p}\bar
                                                       a_{p}^{\trn})^{-1}}- 2 a^{\trn}\left(H_{p}-\bar a_{p}\bar a_{p}^{\trn}\right)^{-1}\bar a_{p}+\norm{\bar a_{p}}^2_{\left(H_{p}-\bar a_{p}\bar a_{p}^{\trn}\right)^{-1}}+1\notag\\
    &=\norm{a-\bar a_{p}}^2_{(H_{p}-\bar a_{p}\bar a_{p}^{\trn})^{-1}}+1\,.\label{eq:norm identity}
\end{align}
By plugging this expression into \pref{eq:logdet_opt3}, it follows that the optimality conditions for the problems \eqref{eq:enhanced} and \eqref{eq:logdetbarrier} are identical. Any solution $\pstar$ to the problem
\eqref{eq:enhanced} yields a solution to the problem
\eqref{eq:logdetbarrier}, and vice-versa.
\end{proof}

In light of \pref{prop:lifted}, we work exclusively with the
lifted problem going forward. Before describing our algorithm, it will be useful to introduce the following approximate version of the
optimality condition in \pref{eq:minimax}, which quantifies the
quality of a candidate solution $p\in\Delta(\cA)$.
\begin{definition}
For any action set $\cA$, parameter $\hat{\theta}\in\bbR^{d}$, and learning rate $\gamma>0$, a distribution $p\in\Delta(\cA)$ is called an \emph{$\eta$-rounding} for $\logdetbarrier(\cA,\hat\theta, \gamma)$ if it satisfies
\begin{align}
\label{eq:eta_rounding}
    \forall a\in\cA\colon\quad \frac{1}{\gamma}\norm{\tilde a}^{2}_{\tilde H_{p}^{-1}}  \leq (1+\eta)\left(\frac{\tilde d}{\gamma}+\ip{\tilde{a}-\tilde a_{p},\tilde\theta} \right)\,.
\end{align}
\end{definition}
The following lemma quantifies the loss in regret incurred by sampling
from an $\eta$-rounding for the $\logdetbarrier$ objective rather than an exact solution.
\begin{lemma}
  \label{lem:rounding_to_regret}
  Suppose that for all steps $t$, we sample from an $\eta$-rounding
  for $\logdetbarrier(\cA_t,\hat\theta_t,\gamma/(1+\eta))$ within
  \pref{alg:logdet}. Then the regret bound from \cref{lem:logdet bound1} increases by at most a factor of $1+2\eta$.
\end{lemma}
\pref{lem:rounding_to_regret} implies that to achieve the regret bound
from \pref{thm:infinite} up to a factor of $2$, it suffices to find a $1/2$-rounding.

\begin{proof}
We first prove an analogue of the inequality in \cref{lem:logdet
  bound1}. Let $t$ be fixed and abbreviate $\thetahat\equiv\thetahat_t$. Assume without loss of generality that $d=\dim(\cA_t)$.
For an $\eta$-rounding $p$ that satisfies \pref{eq:eta_rounding} with learning rate $\gamma'\ldef{}\gamma/(1+\eta)$, 
by the identity \eqref{eq:norm identity} the following inequalities
are equivalent: 
\begin{align*}
      &\frac{1}{\gamma'}\norm{\tilde a}^{2}_{\tilde H_{p}^{-1}}  \leq (1+\eta)\left(\frac{\tilde d}{\gamma'}+\ip{a-\bar a_{p},\hat\theta} \right)\\
    &\iff\frac{1+\eta}{\gamma}\norm{\tilde a}^{2}_{\tilde H_{p}^{-1}}  \leq (1+\eta)\left(\frac{\tilde d(1+\eta)}{\gamma}+\ip{a-\bar a_{p},\hat\theta} \right)\\
  &\iff\frac{1}{\gamma}\left(\norm{a-\bar a_p}^2_{(H_p-\bar a_p\bar a^{\trn}_p)^{-1}}+1\right)\leq \frac{(d+1)(1+\eta)}{\gamma}+\ip{a-\bar a_{p},\hat\theta}\\
  &\iff\ip{\bar a_{p}-a,\hat\theta}+\frac{1}{\gamma}\norm{a-\bar a_p}^2_{(H_p-\bar a_p\bar a^{\trn}_p)^{-1}}\leq \frac{d}{\gamma}\left(1+\eta+\frac{\eta}{d}\right)\,.
\end{align*}
It follows that the bound from \cref{lem:logdet bound1} increases by at
most a
factor of 
$(1+\eta+\frac{\eta}{d})<1+2\eta$ if we use an $\eta$-rounding rather
than an exact solution.
\end{proof}

\subsection{Algorithm}
\paragraph{Preliminaries.}
To keep notation
compact, throughout this section we drop
the learning rate parameter and work with the objective
\begin{align}
    G(p) \ldef \ip{\tilde a_p,\tilde\theta} - \log\det(\tilde H_p),\mathand{}\pstar\in\argmin_{p\in\Delta(\cA)}G(p).\label{eq:Gp2}
\end{align}
Note that this suffices to capture the case where $\gamma\neq{}1$ (\pref{eq:Gp}), since we can multiply both terms
by $\gamma^{-1}$ and absorb a gamma factor into $\theta$. Consequently, for
the remainder of the section we work under the assumption that
$\norm{\theta}\leq \gamma$ rather than
$\nrm{\theta}\leq{}1$. The definition of an $\eta$-rounding  remains
unaffected, since we can multiply both sides in \pref{eq:eta_rounding}
by $\gamma$.

\paragraph{Additional notation.}
For each $a\in\cA$, let $\mathbf{e}_a\in\Delta(\cA)$ be the distribution that selects $a$ with probability $1$. For distributions $p_1,p_2\in\Delta(\cA)$, let $\conv[p_1,p_2]=\{\lambda
p_1+(1-\lambda)p_2\,|\,\lambda\in[0,1]\}$ be their convex hull.
To improve readability, we abbreviate $\norm{\cdot}_{\tilde H^{-1}_p}$ to $\norm{\cdot}_p$ in this section.

\paragraph{Algorithm.}
Our main algorithm is stated in \pref{alg:fw}. The algorithm is a generalization of Khachiyan's algorithm for optimal experimental design \citep{KT90}.
It maintains a finitely supported distribution over arms in $\cA$ and adds a
single arm to the support at each step.

In more detail, the algorithm proceeds as follows. At step $k$, the
algorithm checks whether the current iterate $p_{k-1}$ is an $\eta$-rounding. If
this is the case, the algorithm simply terminates, as we are
done. Otherwise, with $\astar\ldef\argmin_{a\in\cA}\tri{a,\theta}$, the algorithm first checks whether the current distribution satisfies  
$\tilde d + \ip{\astar-\bar a_{p_{k-1}},\theta} \geq 1$. If that
condition is violated, we define a new distribution $p'_{k-1}$ by choosing
the distribution in $\conv[p_{k-1}, e_{\astar}]$ that minimizes
$G(p)$. This ensures that $\frac{\partial}{\partial \lambda}\brk{G(p'_{k-1}+x(\mathbf{e}_{\astar}-p_{k-1})}(0)=0$, i.e.
\begin{align*}
 \ip{a^*,\theta}-\norm{a^*}^2_{p'_{k-1}} = \E_{a\sim p'_{k-1}}\left[\ip{a,\theta}-\norm{a}^2_{p'_{k-1}}\right] = \ip{\bar a_{p'_{k-1}},\theta}-\tilde d\,,
\end{align*}
and hence $\min_{a\in\cA}\tilde d + \ip{a-\bar a_{p'_{k-1}},\theta} =
\norm{a^*}^2_{p'_{k-1}}\geq 1$. In particular, this implies that
\begin{align}
\eta_k \ldef \max_{a\in\cA}\norm{\tilde a}^2_{p'_{k-1}}/(d+\ip{a-\bar a_{p'_{k-1}},\theta})
  \label{eq:etak}
\end{align}
 is well defined. To conclude the iteration, the algorithm selects an
 action $a_k$ that attains the maximum in \pref{eq:etak} and adds it to
 the support of $p'_{k-1}$, yielding $p_k$.

\begin{algorithm}[H]
\caption{Frank-Wolfe for minimizing the $\logdetbarrier$ objective}
\label{alg:fw}
\DontPrintSemicolon
\LinesNumberedHidden
\KwIn{$p_0\in\Delta(\cA),\cA,\theta,\eta$}
Let $\astar = \argmin_{a\in\cA}\ip{a,\theta}$, $k=1$.\;
\While{\text{\textnormal{$p_{k-1}$ is not an $\eta$-rounding  (\pref{eq:eta_rounding})}}}{
\If{$\tilde d + \ip{a^*-\bar a_{p_{k-1}},\theta} < 1$}{
Solve $p_{k-1}' = \argmin_{p\in\conv[p_{k-1},\mathbf{e}_{a^*}]}G(p)$.
}
\Else{
$p_{k-1}'=p_{k-1}$.
}
Pick any $a_k \in \argmax \norm{\tilde a}^2_{p'_{k-1}}/(d+\ip{a-\bar
  a_{p_{k-1}'},\theta})$ (ties broken arbitrarily).\;
Solve $p_{k} = \argmin_{p\in\conv[p_{k-1}',\mathbf{e}_{a_k}]}G(p)$.\;
Increment $k$.
}
\end{algorithm}

\subsection{Analysis}
In this section we prove a number of intermediate results used to
bound the iteration complexity of \pref{alg:fw}, culminating in our
main convergence guarantee, \pref{thm:fw}. The total computational
complexity is summarized at the end of the section in \pref{sec:total_comp}.

We begin by relating the $\eta$-rounding property to the suboptimality
gap for the objective $G(p)$.
\begin{lemma}
  \label{lem:gap}
  If $p\in\Delta(\cA)$ is an $\eta$-rounding, then
\begin{align*}
    G(p)-G(\pstar) \leq \log(1+\eta)\tilde d\,.
\end{align*}
\end{lemma}
\begin{proof}[Proof of \cref{lem:gap}]
By the optimality conditions in \pref{eq:logdet_opt1,eq:logdet_opt2,eq:logdet_opt3}, we are guaranteed that
\begin{align*}
    \forall a\in\supp(\pstar):\, \tilde d +\ip{a,\theta} = \norm{\tilde a}_{\tilde H_{\pstar}^{-1}}^2+\ip{\bar a_{\pstar},\theta}\,.
\end{align*}
Hence, combining this statement with the $\eta$-rounding condition for
$p$, we have that
\begin{align*}
    \forall a\in\supp(\pstar):\, \norm{\tilde a}^2_{\tilde H^{-1}_p}\leq (1+\eta)\left(\norm{\tilde a}_{\tilde H_{\pstar}^{-1}}^2 +\ip{\bar a_{\pstar}-\bar a_p,\theta}\right).
\end{align*}
Taking the expectation over $a\sim \pstar$ on both sides above and rearranging leads to
\begin{align*}
    \ip{\bar a_p-\bar a_{\pstar},\theta} \leq \tilde d-\frac{\tr(\tilde H_{\pstar}\tilde H^{-1}_p)}{1+\eta}=\tilde d-\frac{\tr(\tilde H^\frac{1}{2}_{\pstar}\tilde H^{-1}_p\tilde H^\frac{1}{2}_{\pstar})}{1+\eta}\,.
\end{align*}
From the definition of $G(p)$, this implies that
\begin{align*}
    G(p)-G(\pstar) \leq \tilde d-\frac{\tr(\tilde H^\frac{1}{2}_{\pstar}\tilde H^{-1}_p\tilde H^\frac{1}{2}_{\pstar})}{1+\eta} + \log\det(\tilde H^\frac{1}{2}_{\pstar}\tilde H^{-1}_p\tilde H^\frac{1}{2}_{\pstar}),
\end{align*}
where we recall that $\det(\tilde H^\frac{1}{2}_{\pstar}\tilde H^{-1}_p\tilde H^\frac{1}{2}_{\pstar})=\det(\tilde H_{\pstar}\tilde H^{-1}_p)>0$, since
$\tilde{H}_{\pstar}$, $\tilde{H}_{p}\psdgt{}0$.
Now, let $(\lambda_i)_{i=1,\dots,\tilde d}$ be the eigenvalues of
$\tilde H^\frac{1}{2}_{\pstar}\tilde H^{-1}_p\tilde
H^\frac{1}{2}_{\pstar}$. Then we have
\begin{align*}
    G(p)-G(\pstar) = \sum_{i=1}^{\tilde d} 1-\frac{\lambda_i}{1+\eta}+\log(\lambda_i) \leq \tilde d \max_{\lambda>0}\crl*{ 1-\frac{\lambda}{1+\eta}+\log(\lambda)}=\tilde d \log(1+\eta)\,.
\end{align*}
\end{proof}
Our next lemma lower bounds the rate at which the suboptimality gap
improves at each iteration.
\begin{lemma}
\label{lem:1step convergence}
In each iteration of \cref{alg:fw}, the suboptimality gap improves by
at least
\begin{align}
    G(p_{k-1})-G(p_{k}) \geq \Omega\left(\min\{\eta_k,1\}^2/d\right),\label{eq:1step_1}
\end{align}
where we recall that $\eta_k \ldef \norm{a_k}^2_{p'_{k-1}}/(\tilde d + \ip{a_k-\bar a_{p_{k-1}'}})$.
Furthermore, if $\eta_k\geq 2\tilde d$, then it also holds that
\begin{align}
    G(p_{k})-G(\pstar)\leq \left(1-\frac{1}{2\tilde d}\right)\left(G(p_{k-1})-G(\pstar)\right)\,.\label{eq:1step_2}
\end{align}
\end{lemma}

\begin{proof}
  We first prove that \pref{eq:1step_1} holds. Let $k$ be fixed, and let $\alpha \in [0,1]$ such that
$p_{k} = (1-\alpha)p_{k-1}'+\alpha \mathbf{e}_{a_k}$. Then we have
\begin{align*}
    G(p_k) &= \ip{\bar a_{p_k},\theta} - \log\det\left(\tilde H_{p_k}\right)\\
    &= (1-\alpha)\ip{\bar a_{p_{k-1}'},\theta}+\alpha\ip{\tilde a_k,\theta}-\log\det\left((1-\alpha)\tilde H_{p_{k-1}'}+\alpha \tilde a_k\tilde a_k^{\trn}\right)\\
    &=\ip{\bar a_{p_{k-1}'},\theta}+\alpha\ip{\tilde a_k-\bar a_{p_{k-1}'},\theta}-\log\left(\det \left((1-\alpha)\tilde H_{p_{k-1}'}\right)\cdot\prn*{1+\frac{\alpha}{1-\alpha}\norm{\tilde a_k}^2_{p'_{k-1}}} \right)\\
    &=G(p_{k-1}') + \alpha\ip{\tilde a_k-\bar a_{p_{k-1}'},\theta} - (\tilde d-1) \log(1-\alpha)-\log\left(1-\alpha+\alpha\norm{\tilde a_k}^2_{p'_{k-1}}\right),
\end{align*}
where the third equality uses the matrix determinant lemma. Now,
recall that by the definition of $a_k$, we have $\norm{\tilde
  a_k}^2_{p'_{k-1}}=(1+\eta_k)(\tilde d
+\ip{\tilde{a}_k-\bar a_{p_{k-1}'},\theta})$. Let us abbreviate
$Z_k\ldef\norm{\tilde a_k}^2_{p'_{k-1}}\geq
1+\eta_k$. We proceed as
\begin{align}
  G(p_{k-1})-G(p_{k})&\geq G(p_{k-1}')-G(p_{k}) \notag\\
    &= \alpha\ip{\bar a_{p_{k-1}'}-\tilde{a}_k,\theta} + (\tilde d-1) \log(1-\alpha)+\log\left(1-\alpha+\alpha\norm{\tilde a_k}^2_{p'_{k-1}}\right)\notag\\
    &= \alpha\left(\tilde d-\frac{Z_k}{1+\eta_k}\right) + (\tilde d-1) \log(1-\alpha)+\log\left(1+\alpha( Z_k-1)\right)\notag\\
    &=\max_{\alpha'\in[0,1]} \crl*{\alpha'\left(\tilde d-\frac{Z_k}{1+\eta_k}\right) + (\tilde d-1) \log(1-\alpha')+\log\left(1+\alpha'( Z_k-1)\right)},\label{eq:1step_max}
\end{align}
where the last equality uses that $\alpha$ is chosen such that $G(p_k)$ is minimized.
Next, recalling the elementary fact that for all $x\geq-\frac{1}{2}$,
$\log(1+x) \geq x-x^2$, we have in particular that
\begin{align*}
&G(p_{k-1})-G(p_{k})\\
&\geq \max_{\alpha'\geq \frac{1}{2}}\crl*{\alpha'\left(\tilde d-\frac{Z_k}{1+\eta_k}\right) + (\tilde d-1)(-\alpha'-\alpha'^2)+\alpha'( Z_k-1)-\alpha'^2( Z_k-1)^2}\\
&= \max_{\alpha'\geq \frac{1}{2}}\crl*{\alpha'\frac{\eta_k Z_k}{1+\eta_k}-\alpha'^2\left(\tilde d-1+(Z_k-1)^2\right)}\,.
\end{align*}
Note that $\tilde d\geq 3$ and $\max_{x>0}\frac{x}{2+(x-1)^2}\leq 1$,
so if we choose
\begin{align*}
    \alpha' = \frac{\eta_kZ_k}{2(1+\eta_k)\left(\tilde d-1+(Z_k-1)^2\right)}\leq \frac{1}{2}\,,
\end{align*}
we get the lower bound
\begin{align*}
    G(p_{k-1})-G(p_{k})\geq \frac{\eta_k^2Z_k^4}{4(1+\eta_k)^2\left(\tilde d-1+(Z_k-1)^2\right)}\,.
\end{align*}
The proof of \pref{eq:1step_1} now follows by noting that
$\frac{x^2}{d+(x-1)^2}\geq \frac{1}{d}$ for all $x\geq 1$.

We now prove that the second part of the lemma, \pref{eq:1step_2},
holds. Suppose $\eta_k>2\tilde{d}$. We return to
\pref{eq:1step_max} and this time select
\begin{align*}
    \alpha' &= \frac{\sqrt{\eta_k}}{Z_k-1}
    \leq \frac{1}{\sqrt{\eta_k}}\leq\frac{1}{2}\,.
\end{align*}
Using the approximation $\log(1+x)\geq{}x-x^{2}$ only for the first
term in \eqref{eq:1step_max}, we get
\begin{align*}
    G(p_{k-1})-G(p_{k})&\geq
    \alpha'\left(\tilde d-\frac{Z_k}{1+\eta_k}\right) - (\tilde d-1) (\alpha'+\alpha'^2)+\log\left(1+\alpha'( Z_k-1)\right)\\
    &\geq-\frac{\sqrt{\eta_k}}{1+\eta_k}-\frac{\tilde
      d-1}{\eta_k}+\log(1+\sqrt{\eta_k})\\
                       &=-\frac{\sqrt{\eta_k}}{1+\eta_k}-\frac{\tilde d-1}{\eta_k}+\log(1+\sqrt{\eta_k})-\frac{1}{4}\log(1+\eta_k)+\frac{1}{4}\log(1+\eta_k)\\
                       &\geq -\frac{\sqrt{\eta_k}}{1+\eta_k}-\frac{1}{2}+\frac{1}{\eta_k}+\log(1+\sqrt{\eta_k})-\frac{1}{4}\log(1+\eta_k)+\frac{1}{4}\log(1+\eta_k)\,,
\end{align*}
where the last line uses that $\eta_k\
\geq 2\tilde d$. Now observe that for $x\geq 6$
\begin{align*}
    &\frac{\partial}{\partial x}\left(-\frac{\sqrt{x}}{1+x}+\frac{1}{x}+\log(1+\sqrt{x})- \frac{1}{4}\log(1+x)\right)\\
    &=\frac{x-1}{2\sqrt{x}(1+x)^2}-\frac{1}{x^2}+\frac{1}{2(\sqrt{x}+x)}-\frac{1}{4(1+x)}\\
    &=\frac{x^\frac{7}{2}+x^3+5x^\frac{5}{2}-7x^2-12x^\frac{3}{2}-8x-4x^\frac{1}{2}-4}{4x^2(1+\sqrt{x})(1+x)^2}\\
    &\geq\frac{7x^2+60x^\frac{3}{2}-7x^2-12x^\frac{3}{2}-8x-4x^\frac{1}{2}-4}{4x^2(1+\sqrt{x})(1+x)^2}
    \geq{}0\,.
\end{align*}
Hence
\begin{align*}
    &-\frac{\sqrt{\eta_k}}{1+\eta_k}-\frac{1}{2}+\frac{1}{\eta_k}+\log(1+\sqrt{\eta_k})-\frac{1}{4}\log(1+\eta_k)\\
    &\geq -\frac{\sqrt{6}}{1+6}-\frac{1}{2}+\frac{1}{6}+\log(1+\sqrt{6})-\frac{1}{4}\log(1+6) > 0\,.
\end{align*}
It follows that
\begin{align*}
    G(p_{k-1})-G(p_{k})\geq{} \frac{1}{4}\log(1+\eta_k).
\end{align*}
\end{proof}
The next lemma ensures we can efficiently find a good initial
distribution $p_0$.
\begin{lemma}[\citet{KY05}, Lemma 3.1]
\label{thm:init}
There exists an algorithm that terminates in $\cO(|\cA|d^2)$ time and finds a distribution $p_0\in\Delta(\cA)$ with $|\supp(p_0)|\leq 2\tilde d$ such that
\begin{align*}
    -\log\det(\tilde H_{p_0})+\min_{p\in\Delta(\cA)}\log\det(\tilde H_p) =\cO(d\log(d))\,.
\end{align*}
The memory required by the algorithm is at most $\cO\left(d^2+\log(|\cA|d)\right)$.
\end{lemma}
\begin{corollary}
\label{cor:init}
The distribution $p_0$ described in \cref{thm:init} has initial suboptimality gap at most \begin{align*}
    G(p_0)-G(\pstar) = \cO(d\log(d)+\gamma)\,.
\end{align*}
\end{corollary}
\begin{proof}
  Recall that
  \[
    G(p_0) - G(\pstar) = \tri{\bar{a}_{p_0}-\bar{a}_{\pstar},\theta}
    -\log\det(\tilde{H}_{p_0}) + \log\det(\tilde{H}_{\pstar}).
  \]
The difference between the log-det terms is bounded by
$\bigoh(d\log(d))$ using
\cref{thm:init}, while the difference between the linear terms is bounded by
\begin{align*}
    \ip{\bar a_{p_0}-\bar a_{\pstar},\theta}\leq \norm{\bar a_{p_0}-\bar a_{\pstar}}\cdot\norm{\theta}\leq 2\gamma\,.
\end{align*}
\end{proof}

\begin{theorem}
  \label{thm:fw}
If \cref{alg:fw} is initialized using the distribution from \cref{thm:init}, then it requires
$\cO\prn*{d(\log(d) + \log(\gamma)}$ iterations to reach a
$2d$-rounding. Moreover, 
\begin{itemize}
\item After reaching the $2d$-rounding above, the algorithm requires $\cO\prn*{\log(d)d^2}$
  additional iterations to reach a $1$-rounding.
\item After reaching a $1$-rounding, the algorithm requires
  $\cO\prn*{d^2/\eta}$ additional iterations to reach an
  $\eta$-rounding for any $\eta<1$.
\end{itemize}
Altogether, for any $\eta>0$, \pref{alg:fw}---when initialized
using \pref{thm:init}---requires \[\bigoh(d\log(\gamma) +
  d^{2}(\log(d)+1/\eta))\]
total steps to reach an $\eta$-rounding.
\end{theorem}
\begin{proof}
By \cref{cor:init}, the initial distribution $p_0$ satisfies
\begin{align*}
    G_0 \ldef G(p_0)-G(\pstar) = \cO(d\log(d)+\gamma)\,.
\end{align*}
We first bound the number of steps required to reach a
$2d$-rounding. 
Let $k_0$ denote the first step $k$ in which $p_{k}$ is a
$2d$-rounding. Then every $k<k_0$ has $\eta_k > 2d$, so in light of
\cref{lem:1step convergence}, all such $k$ have
\[
  G(p_k)-G(p_0)\leq\prn*{1-\frac{1}{2\tilde{d}}}(G(p_{k-1}) - G(p_0))
\]
and
\[
  G(p_k) \leq{} G(p_{k-1}) - \bigom(1/d).
\]
It follows that as long as $\eta_k>2d$, the suboptimality gap will
reach $1$ in most $\cO\prn*{d\log(G_0)}=\bigoh(d(\log(d)+\log(\gamma)))$ iterations. Moreover, since
the absolute decrease in function value is at least $\Omega(1/d)$, the
gap would reach zero after another $\cO(d)$ iterations of this type. We conclude
that after $\cO\prn*{d(\log(d)+\log(\gamma))}$ iterations, the
algorithm must find a $2d$-rounding.

We now bound the number of steps to reach a $1$-rounding from the
first step where we have a $2d$-rounding. By \cref{lem:gap}, the suboptimality gap
of any $2d$-rounding is at most $\cO\prn{d\log(d)}$. Moreover, as long
as we haven't reached a $1$-rounding, \pref{lem:1step convergence}
guarantees that the suboptimality gap will decrease by $\Omega(1/d)$
per step. Hence, we must reach a $1$-rounding within $\cO\prn{d^2\log(d)}$ iterations.

Finally we bound the number of steps required to reach an
$\eta$-rounding for any $\eta<1$, starting from the first iteration where
we reach a $1$-rounding. We adapt an argument of
\citet{KY05}. Given an $\eta_k$-rounding for $\eta_k\leq 1$, we need
$\cO\prn*{d^2/\eta_k}$ iterations to reach an
$(\eta_k/2)$-rounding. This follows from the same argument as above: the
suboptimality gap is at most $\cO(d\eta_k)$ by \pref{lem:gap} (using
that $\log(1+\eta_k)\leq\eta_k$) and we reduce it by
$\Omega(\eta_k^2/d)$ as long as we have not found an
$(\eta_k/2)$-rounding (by \pref{lem:1step convergence}).
Summing up the required number of iterations to get from precision $1$ to $1/2$ to $1/4$ to \dots
to $1/2^{\ceil{\log_2(1/\eta)}}$ shows that $\cO(d^2/\eta)$
total iterations suffice.
\end{proof}

\subsubsection{Total Computational Complexity}
\label{sec:total_comp}
The computational complexity per iteration for our method is comparable to similar
algorithms for the D-optimal design problem, which we recall is the case where $\theta=0$
\citep{KT90,KY05,TY07}. We walk calculatethe complexity
step-by-step for completeness, and to handle differences arising from
our generalization to the $\theta\neq{}0$ case. The first difference is that our method solves an intermediate optimization problem over the line $\conv(p_{k-1},\mathbf{e}_{\astar})$. This step increases the computational complexity by a factor of $2$. At each iteration, \cref{alg:fw} computes
\begin{align*}
    \argmax_{a\in\cA} \frac{\norm{\tilde a}^2_{p'_{k-1}}}{d+\ip{a-\bar a_{p'_{k-1}},\theta}}\,.
\end{align*}
For generic action sets, this can be computed in time $\cO(|\cA|d^2)$, given that $\tilde
H_{p_{k-1}'}^{-1}$ has already been computed. In the next step, the algorithm
solves the one dimensional optimization problem
\begin{align*}
    \max_{\alpha'\in[0,1]} \left(\alpha'\left(\tilde d-\frac{Z_k}{1+\eta_k}\right) + (\tilde d-1) \log(1-\alpha')+\log\left(1+\alpha'( Z_k-1)\right)\right)\,,
\end{align*}
where $Z_k=\norm{\tilde a_k}^2_{p'_{k-1}}$.
This can be done in time $\cO(1)$, since it is equivalent to solving the quadratic problem
\begin{align*}
    \left(\tilde d-\frac{Z_k}{1+\eta_k}\right) - \frac{\tilde d-1}{1-x}+\frac{Z_k-1}{1+x( Z_k-1)}=0\,.
\end{align*}

Finally we need to update $\bar a_p$, which costs $\cO(d)$, and update
$\tilde H^{-1}_p$, which can be done in time $\bigoh(d^{2})$ using a
rank-one update.

Across all iterations, we require a total of 
$\tilde\cO(d^4|\cA|)$ arithmetic operations, with $p_k$ never exceeding support size $\cO\prn*{d^2\log(d)+d\log(\gamma)}$, since we add at most one arm to the support in any iteration.
We can store $p_k$ as a sparse vector of key and value pairs, where each entry uses memory complexity of $\cO\prn*{\log(|\cA|)}$ to represent the key.

\fi

\end{document}